\documentclass[journal, onecolumn]{IEEEtran}
\usepackage{graphicx}
\usepackage{float}
\usepackage{hyperref}
\usepackage{multirow}
\usepackage{algorithm}% http://ctan.org/pkg/algorithms
\usepackage{algpseudocode}% http://ctan.org/pkg/algorithmicx
\usepackage{ amssymb, amsmath,amsthm}
\usepackage{cleveref}
\usepackage{bbm}
\usepackage{xcolor}
\usepackage{array}
\newtheorem{proposition}{Proposition}
\numberwithin{proposition}{section}
\newtheorem{lemma}[proposition]{Lemma}
\newtheorem{theorem}[proposition]{Theorem}

\newtheorem{corollary}[proposition]{Corollary}
\newtheorem{claim}[proposition]{Claim}

\usepackage{thmtools}
\usepackage{thm-restate}
\newtheorem{definition}[proposition]{Definition}
\newtheorem{remark}[proposition]{Remark} 
\newtheorem{question}{Question}

\DeclareMathOperator*{\argmin}{arg\,min}
\DeclareMathOperator{\Tr}{Tr}
\DeclareMathOperator{\Aut}{Aut}

\DeclareMathOperator{\Cov}{Cov}

\newcommand{\Astar}{\alpha^*}
\newcommand{\Abar}{\overline{\alpha}}
\newcommand{\Abarstar}{\overline{\Astar}}
\newcommand{\nf}[2]{\nicefrac{#1}{#2}}
\usepackage{tikz,nicefrac}
\usetikzlibrary{arrows,decorations.pathmorphing,decorations.shapes,snakes,patterns,positioning}
\usepackage{multirow}

\newcommand{\gammaG}{\gamma G}
\usepackage{setspace}
\usepackage{makecell}

\title{Approximate Gradient Coding with Optimal Decoding}

\begin{document}

\author{Margalit Glasgow, Mary Wootters~\IEEEmembership{Member,~IEEE}
\thanks{M. Glasgow was with the Department
of Computer Science, Stanford University, Stanford,
CA, 94305 USA e-mail: \href{mailto:mglasgow@stanford.edu}{mglasgow@stanford.edu}.}% <-this % stops a space
\thanks{M. Wootters was with the Department
of Computer Science and Department of Electrical Engineering, Stanford University, Stanford,
CA, 94305 USA e-mail:\href{mailto:marykw@stanford.edu}{marykw@stanford.edu} .}% <-this % stops a space
\thanks{\textcopyright 2021 IEEE. Personal use of this material is permitted. Permission from IEEE must be
obtained for all other uses, in any current or future media, including
reprinting/republishing this material for advertising or promotional purposes, creating new
collective works, for resale or redistribution to servers or lists, or reuse of any copyrighted
component of this work in other works.}
\thanks{This paper was published in IEEE Journal on Selected Areas in Information Theory, available at \href{https://doi.org/10.1109/JSAIT.2021.3100110}{https://doi.org/10.1109/JSAIT.2021.3100110}}}

\maketitle

\begin{abstract}
Gradient codes use data replication to mitigate the effect of straggling machines in distributed machine learning. \textit{Approximate gradient codes} consider codes where the data replication factor is too low to recover the full gradient exactly. Our work is motivated by the challenge of designing approximate gradient codes that simultaneously work well in \textit{both} the adversarial and random straggler models. We introduce novel approximate gradient codes based on expander graphs. We analyze the decoding error both for random and adversarial stragglers, when optimal decoding coefficients are used. With random stragglers, our codes achieve an error to the gradient that decays exponentially in the replication factor. With adversarial stragglers, the error is smaller than any existing code with similar performance in the random setting. We prove convergence bounds in both settings for coded gradient descent under standard assumptions. With random stragglers, our convergence rate improves upon rates obtained via black-box approaches. With adversarial stragglers, we show that gradient descent converges down to a noise floor that scales linearly with the adversarial error to the gradient. We demonstrate empirically that our codes achieve near-optimal error with random stragglers and converge faster than algorithms that do not use optimal decoding coefficients.
\end{abstract} 

\section{Introduction}
    
    Consider the task of minimizing some loss function $L$ summed over $N$ data points $\{(x_i, y_i)\}_{i = 1}^N$: $$\min_{\theta}\sum_{i = 1}^N{L(x_i, y_i; \theta)}.$$ When $N$ is large, we can parallelize the computation of the gradient of this function by distributing the data points among $m$ worker machines, as has become common practice for large-scale machine learning problems \cite{li2014communication}. Each machine computes the gradient of the functions available to it and returns the sum of these gradients to the parameter server. Recent work has pointed out the prevalence of \textit{stragglers}, i.e. machines that are slow or unresponsive, which can significantly slow down the execution of distributed computing tasks such as synchronous gradient descent~\cite{dean2013tail, zaharia2008improving}.
To mitigate this effect, previous work has used a technique called \textit{gradient coding}, which involves replicating each data point and sending it to multiple machines \cite{Tandon}. While this increases the computation load and storage at each machine, it has the potential to speed up convergence by allowing the parameter server to compute an exact or closer approximation to the true gradient, even in the presence of stragglers. 

    % Our main contribution, as we elaborate on in Section~\ref{ssec:contributions} below, is the construction of approximate gradient coding schemes that perform well in settings where the stragglers are chosen either randomly or adversarially.  In the \em approximate gradient coding \em setting, introduced in \cite{ExpanderCode}, the data replication factor is too low to compute an exact gradient at the parameter server.  Previous work in this setting has developed good schemes under the assumption of random stragglers~\cite{Tandon, ErasureHead,limits}, and, \em separately, \em good schemes under the assumption of adversarial (worst-case) stragglers~\cite{ExpanderCode}.  However, existing schemes that are good for one assumption are not good for the other.  Our schemes offer the best of both worlds, performing well under either assumption, in both theory and practice.

In a typical setting of gradient coding (e.g. \cite{PB, ExpanderCode}), we let $A \in \mathbb{R}^{N \times m}$ be an \em assignment matrix \em of data points to machines, such that $A_{ij} \neq 0$ if and only if the $i$th data point is held by machine $j$. We define the \em replication factor \em of an assignment as follows.

    \begin{definition}[Replication Factor] 
    The \em replication factor \em of an assigment matrix $A \in \mathbbm{R}^{N \times m}$ is the average number of times a data point is replicated, that is, the number of non-zero entries in $A$ divided by $N$.
    \end{definition}

    In coded gradient descent, at each round $t$ of computation, the parameter server broadcasts the current point $\theta_t$ to the machines.  Each non-straggling machine $j$ returns the single vector $$g_j := \sum_{i = 1}^N{A_{ij}\nabla f_i(\theta_t)},$$ to the parameter server, where we have defined $f_i(\theta) := L(x_i, y_i; \theta)$. The parameter server then chooses some decoding coefficient vector $w \in \mathbb{R}^m$, where $w_j = 0$ if machine $j$ straggles, and performs the update \begin{equation}\label{star}
            \theta_{t + 1} \leftarrow \theta_t - \gamma\sum_{j = 1}^m{w_j g_j}
        \end{equation}
        for some learning rate $\gamma$. For any coefficient vector $w$, we define $\alpha := Aw,$ such that the update in \Cref{star} can be written
    \begin{equation}\label{star2}
         \theta_{t + 1} \leftarrow \theta_t - \gamma\sum_{i = 1}^N{\alpha_i\nabla f_i(\theta_t)}.
    \end{equation}
    If a coding scheme---that is, a matrix $A$ and a way of computing the coefficients $w$---can always achieve $\alpha = \mathbbm{1}$, then it recovers the gradient exactly, and \Cref{star} can be analyzed as full-batch gradient descent. While this is ideal, it often requires an assignment matrix with a high replication factor. If we cannot recover the full gradient exactly, we are in the case of approximate gradient coding.

   Most previous work on approximate gradient coding has fallen into one of two categories. In one line of work, the non-zero coefficients $w_j$ are fixed in advance (\cite{PB}) or only depend on the number of stragglers  (\cite{ExpanderCode, bibd}). In particular, these non-zero coefficients do not depend on the identity of the stragglers. A second line of work (\cite{Charles, limits, ErasureHead}) chooses the decoding coefficients $w$ dynamically depending on which machines straggle. This is called \textit{optimal decoding}\footnote{It's not necessarily the case the ``optimal decoding'' coefficient lead to optimal convergence.  However, we use the term to be consistent with the literature \cite{Charles,bibd, sbm, ErasureHead}.} because the parameter server chooses $w$ to be any vector
   \begin{equation}\label{eq:optimal}
       w^* \in \argmin_{w : w_j = 0 \text{ if machine j straggles}}{|Aw - \mathbbm{1}|_2},
   \end{equation}
   where $|\cdot|_2$ denotes the $2$-norm of a vector. In optimal decoding, we will denote $\Astar := Aw^*$. 
   % There are two strategies for choosing the coefficients $w$.  The first, \em fixed coefficient decoding, \em either fixes $w$ in advance~\cite{PB} or chooses it based only on the \em number \em of stragglers~\cite{ExpanderCode,bibd}.  The second, \em optimal% 
% \footnote{It is not necessarily the case the ``optimal decoding'' coefficient lead to optimal convergence.  However, we use the term to be consistent with the literature \cite{Charles,bibd, sbm, ErasureHead}.} 
% coefficient decoding\em~\cite{Charles, limits, ErasureHead},  chooses
%     \begin{equation}\label{eq:optimal}
%         w^* \in \argmin_{w : w_j = 0 \text{ if machine j straggles}}{|Aw - \mathbbm{1}|_2}.
%     \end{equation}
% We will let $\Astar = Aw^*$. 

   In this work we will study gradient coding schemes with optimal decoding. The following formalizes the two definitions of decoding error we study. Let $[m]$ denote the set of integers from $1$ to $m$.

    \begin{definition}[Decoding Error under Random Straggler]
    Given an assigment matrix $A \in \mathbb{R}^{N \times m}$, we define the random decoding error under a $p$ fraction of random stragglers to be
    $$\mathbb{E}_S\left[|\Astar - \mathbbm{1}|_2^2\right] = \mathbb{E}_S\left[\min_{\substack{w \in \mathbb{R}^m \\ w_j = 0 \forall j \in S}}\left|Aw - \mathbbm{1}\right|_2^2\right]$$
    where $S$ is a random subset of $[m]$ that includes each value with probability $p$.
    \end{definition}

    In all future instances, we will omit the subscript $S$ and take $\mathbb{E}$ to mean the expectation over the random set of stragglers. 

    \begin{definition}[Decoding Error under Adversarial Straggers]
    Given an assigment matrix $A \in \mathbb{R}^{N \times m}$, we define the adversarial decoding error under a $p$ fraction of adversarial stragglers to be
    $$\max_{S \subset [m]: |S| \leq pm}\left[|\Astar - \mathbbm{1}|_2^2\right] = \max_{S \subset [m]: |S| \leq pm}\left[\min_{\substack{w \in \mathbb{R}^m \\ w_j = 0 \forall j \in S}}\left|Aw - \mathbbm{1}\right|_2^2\right].$$
    \end{definition}

Both random stragglers and adversarial stragglers arise in practice. While random stragglers may arise due to system level variabilities (such as maintainance activities)~\cite{dean2013tail}, adversarial stragglers may arise due to hardware differences among machines or in settings where gradients are slower to compute at some data points.

% There are two models for stragglers. In the \em random \em setting, each machine is independently chosen with probability $p$ to be a straggler. In the \em adversarial \em setting, $\lfloor{pm}\rfloor$ machines are chosen adversarially to be stragglers. 
%%
The main objective in approximate gradient coding is to design assignment matrices $A$ with a small replication factor and a small decoding error in the presence of stragglers. The work \cite{Charles} showed that a particular fractional repetition code (FRC) introduced by \cite{Tandon} achieves the optimal decoding error with random stragglers, over all assigment matrices with the same replication factor. However, the FRC of \cite{Tandon} performs poorly over adversarially chosen stragglers relative to other assignments with the same replication factor. This motivates the main question behind our work:

    \begin{question}\label{question}
    Are there gradient codes that simultaneously achieve small decoding error under both random and adversarial stragglers?
    \end{question}

     As pointed out in the open questions of \cite{Charles}, this question is challenging because of the difficulty of analyzing the decoding error under random stragglers. Indeed, bounding the decoding error with optimal decoding amounts to analyzing the pseudoinverse of the random matrix generated by removing a random set of columns (corresponding to straggling machines) from the assignment matrix.\footnote{Formally, $\Astar = A(p)(A(p)^TA(p))^{\dagger}A(p)^T\mathbbm{1}$, where $A(p)$ is the matrix obtained by deleting each column of the assignment matrix $A$ with probability $p$.  Here, for a matrix $M$, $M^\dagger$ denotes the Moore-Penrose pseudoinverse of $M$. } This is particularly challenging with this random matrix is sparse, which arises when the replication factor is small.

\subsection{Contributions}\label{ssec:contributions}
In this paper we develop schemes that achieve small decoding errors in both the random and an adversarial model simultaneously.
To ensure that gradient descent will converge to the minimum of $f := \sum_i{f_i}$ in the random straggler setting, we construct codes that yield an unbiased approximation of the gradient. That is, when each machine is chosen independently to be a straggler with probability $p$, $$\mathbb{E}{\sum_i{\alpha_i\nabla f_i(\theta_t)}} = c\nabla f(\theta_t)$$ for some constant $c$. In such unbiased schemes, where $\mathbb{E}[\alpha] = c\mathbbm{1}$, we will define $\Abar := \frac{\alpha}{c}$. Here, and in rest of this paper, we use the asymptotic notation big-O and little-o to denote limiting behavior as $d$ goes to $\infty$: We say that $f(d) = O(g(d))$ if $\limsup_{d \rightarrow \infty} \frac{|f(d)|}{g(d)} < \infty$ and $f(d) = O(g(d))$ if $\limsup_{d \rightarrow \infty} \frac{f(d)}{g(d)} = 0$.
With this notation, our contributions are as follows. 
\begin{figure}[b]
  \input{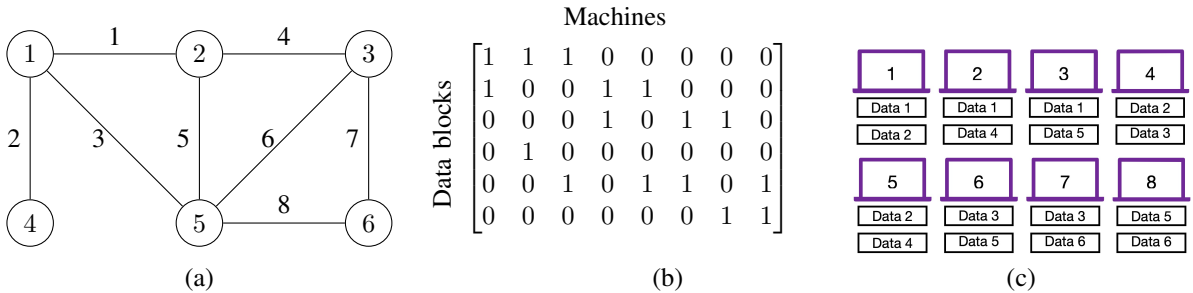} 
  \caption{The assignment generated from the pictured graph. (a) A graph $G$, where the vertices of $G$ correspond to data blocks and the edges correspond to machines. (b) The assignment matrix $A$ generated from $G$. (c) The distribution of data blocks to machines.}
  \label{fig:asst_graph_example}
\end{figure}
\begin{enumerate}
    \item \textbf{A new approach to analyzing optimal coefficient decoding.} 
    While in general analyzing the optimal decoding coefficients is difficult, we develop a framework in which it is tractable.  More precisely, we construct matrices $A$ from a graph $G$ by viewing the data blocks as vertices of $G$, and the machines as edges of $G$, holding two data blocks each.  (See \Cref{graphmatrix} and Figure~\ref{fig:asst_graph_example}). % we note that this is different from the bipartite graph whose bi-adjacency matrix is $A$ which is used in other works~\cite{ExpanderCode, limits}
    For a desired replication factor $d$, we partition the data into blocks of size $\frac{dN}{2m}$, and assign each machine exactly two blocks.

    In both the random and the adversarial case, we relate the decoding error to the \em spectral expansion \em of the graph $G$, defined as the gap between the largest and second largest eigenvalues of the adjacency matrix of $G$. In particular, in the random case we are able to analyze the optimal decoding coefficients by considering random sparsifications of this graph.
    
    Because of the structure of $A$, in our framework we can compute the optimal decoding coefficients $w^*$ in $c\times m$ operations, where $c$ is a universal constant. This is on the same order of the number of operations for the parameter server to compute the update in \Cref{star}.
    \item \textbf{Progress on \Cref{question}.} Using our framework, we construct assignment schemes based on expander graphs (graphs with large spectral expansion) that achieve the following bounds in both the random and adversarial settings.
    \begin{itemize}
        \item In the random setting with optimal decoding, we show in \Cref{expmainfull} that the decoding error decays exponentially in the replication factor $d$:  $\frac{1}{N}\mathbb{E}\left[{\left|\Astar - \mathbbm{1}\right|_2^2}\right] = p^{d - o(d)}.$ This nearly matches the lower bound of $p^{d}/(1 - p^{d})$ (see \Cref{lb}) up to the $o(d)$ term in the exponent. 
        In comparison, for all coding schemes with fixed decoding coefficients, we show in \Cref{naive} that the error decays at best like $1/d$:
         $\frac{1}{N}\mathbb{E}\left[{\left|\Abar - \mathbbm{1}\right|_2^2}\right] \geq \frac{p}{d(1 - p)}.$
        \item In the adversarial setting, for any choice of $\lfloor{pm}\rfloor$ stragglers, we show in \Cref{adversarial_good_expander} that our coding schemes achieve $\frac{1}{N}{\left|\Astar - \mathbbm{1}\right|_2^2} \leq \frac{1 - o(1)}{2}\frac{p}{1 - p}.$ For small $p$, this is nearly a factor of two improvement over the FRC of \cite{Tandon}. %which achieves $$\frac{1}{N}{\left|\Astar - \mathbbm{1}\right|_2^2} = p,$$ for some choice of $pm$ stragglers.
        \end{itemize}
    \item \textbf{Provable convergence with random stragglers.} With random stragglers, our assignment schemes yield good convergence rates under reasonable assumptions about the $f_i$.  This is because we obtain an unbiased approximation of the gradient and can additionally bound the norm of the covariance matrix of $\Astar$.  In particular,
        we show in \Cref{conv_full} that as the desired convergence threshold $\epsilon$ approaches $0$, the dominant term in the number of iterations of coded gradient descent required is $\frac{\log(1/\epsilon)}{\epsilon}p^{d - o(d)}$, where the little-o hides constant factors that depend on $p$ and the functions $f_i$.
         We also provide a black-box\footnote{By black-box methods, we mean methods that only leverage the variance of the gradient update and none of its other statistical properties.} method to debias any coding scheme for random stragglers, that is, given any coding scheme for random stragglers, our tool allows us to convert it to an unbaised scheme without needing to know the inner working of the code. This implies that any further progress on \Cref{question} will yield convergence bounds on gradient descent (see \Cref{conv_biased} in Appendix~\ref{apx:debias}).
    \item \textbf{Provable convergence with worst-case stragglers.} 
    With adversarial stragglers, it is not possible to guarantee convergence to the minimizer $\theta^*$ of $f = \sum_{i=1}^n f_i$. However, if the strong convexity of $f$ is larger than the product of the adversarial decoding error $|\Astar - \mathbbm{1}|_2^2$ and the maximum Lipshitz constant of any $\nabla f_i$, then we can guarantee that coded gradient descent converges down to some noise floor. We show that this noise floor scales linearly with the adversarial quantity  $|\Astar - \mathbbm{1}|_2^2$. More precisely, we show in \Cref{prop:informal_adv_conv} that we can converge to a floor of $|\theta_t - \theta_*|_2^2 \leq O\left(|\Astar - \mathbbm{1}|_2^2\right),$ where the big-O hides constant factors that depend on the functions $f_i$. To our knowledge, this is the first provable convergence guarantee for adversarial stragglers in coded gradient descent. Previous works have obtained adversarial bounds on $|\Astar - \mathbbm{1}|_2^2$ without establishing convergence results \cite{ExpanderCode, limits, bibd}; \Cref{prop:informal_adv_conv} also implies convergence results for these works as well.
    \item \textbf{Empirical Success.} Our algorithm produces good non-asymptotic results. In \Cref{sec_simulations}, we demonstrate empirically that in the random straggler setting, the expected error $\mathbb{E}\left[\left|\Abarstar - \mathbbm{1}\right|_2^2\right]$ in our schemes nearly meets the lower bound of $p^{d}/(1 - p^{d})$. We also show that gradient descent converges in fewer iterations using optimal decoding with our scheme than when using fixed coefficient decoding, and in over $d$ times fewer iterations than an uncoded approach which simply ignores stragglers. In particular, after $50$ iterations of our algorithm with a replication factor of $3$, we observe at least a $\frac{1}{3p^2}$ improvement in mean squared error over fixed coefficient decoding, and at least a $\frac{1}{10p^2}$ improvement in mean squared error over an uncoded approach after $150$ iterations. We observe that our approach converges at the same rate or faster than the state-of-the-art approaches of \cite{ErasureHead} and \cite{ExpanderCode}.
    \end{enumerate}
\subsection{Related Work}
\begin{table}
\begin{center}
\begin{tabular}{
|>{\centering\arraybackslash}p{1.2in}|>{\centering\arraybackslash}p{1in}|>{\centering\arraybackslash}p{1in}|>{\centering\arraybackslash}p{0.8in}|>{\centering\arraybackslash}p{1.5in}| }
 \hline
 Coding Scheme & Decoding Coefficients  & $\frac{1}{N}\mathbb{E}\left[|\alpha - \mathbbm{1}|_2^2\right]$ & 
 Worst Case  $\frac{1}{N}|\alpha - \mathbbm{1}|_2^2$  &  Convergence Proof? \\
 \hline \hline
Expander Code (Cor. 23 \cite{ExpanderCode}) & Fixed &  -  & $< \frac{4p}{d(1 - p)}$ & Yes (random stragglers) \\
\hline
Pairwise Balanced (\cite{PB})   & Fixed & $\geq \frac{p}{d(1 - p)}$ (by \Cref{naive}) & -  & Yes (random stragglers) \\
\hline
BIBD ( Const. 1 \cite{bibd}) & Fixed and Optimal & - & \begin{minipage}{1in} \begin{center} $O(\frac{1}{\sqrt{m}})$ \\ \footnotesize $ d = \Omega(\sqrt{m})$ \end{center} \end{minipage}& No \\
\hline
BRC (\cite{limits}) & Optimal & 
\begin{minipage}{1.5in}\begin{center}
$e^{-\Theta(d)}$ \end{center}\end{minipage}& - & No \\
\hline
rBGC (\cite{Charles}) & Fixed & $< \frac{1}{(1 - p)d}$ & - & No \\
\hline
FRC of \cite{Tandon} (and \cite{ErasureHead}) & Optimal & $p^{d}$ & $p$ & Yes (random stragglers) \\ 
\hline
\Cref{expmainfull}, \Cref{singular} & Optimal & $p^{d - o(d)}$ & $\frac{(1 + o(1))p}{2(1 - p)}$ & Yes (both random and adversarial stragglers) \\
 \hline
\end{tabular}
\vspace{.3cm}
\caption{Comparison of Related Work. The column containing the quantity $\mathbb{E}\left[|\alpha - \mathbbm{1}|_2^2\right]$ is in expectation over a $p$ fraction of random stragglers. The column containing the quantity $|\alpha - \mathbbm{1}|_2^2$ is the worse case value over a $p$ fraction of stragglers. For unbiased coding schemes, the results pertain to $\Abar$ instead of $\alpha$.}
\label{comp_table}
\end{center}
\end{table}
Gradient coding techniques for distributed optimization were first considered in \cite{Tandon}, where some assignment schemes based on \textit{fractional repetition codes} (FRC) were used to recover the \textit{exact} gradient under \textit{worst-case} stragglers. In the particular FRC used by \cite{Tandon}, the data points and the machines are each partitioned into an equal number of disjoint blocks, and each machine in a block receives all the data points in the corresponding block of data points. This body of work on gradient coding was continued in \cite{Li,halbawi2018improving,sbm, draco} and \cite{Ye}, which established the exact trade-off between computation load, worst-case straggler tolerance, and communication complexity.

A line of work (\cite{ExpanderCode,ErasureHead, PB, LDPC,Charles,bibd}) initiated by \cite{ExpanderCode} explores the landscape of \textit{approximate gradient coding}, where the gradient is not recovered exactly. The work \cite{ExpanderCode} considers both exact and approximate gradient codes. The approximate gradient codes in \cite{ExpanderCode} are based on regular expander graphs, and the non-zero decoding coefficients $w$ are fixed up to the number of stragglers. They achieve a decoding error that decays like $1/d$ in the replication factor $d$, even when the straggling machines are chosen adversarially.\footnote{This follows by using a Ramanujan expander in Corollary 23 of \cite{ExpanderCode}.} They then relax the assumption of adversarially chosen stragglers, and bound the convergence of their coded gradient descent under random stragglers, showing that the run time decreases inversely with $d$. The work of \cite{PB} combines \textit{pair-wise balanced} coding schemes with a tight convergence analysis to yield convergence times that decay like $1/d$; that work also uses fixed decoding coefficients $w$. The work \cite{bibd} considers the problem of approximate gradient recovery when the straggling machines are chosen adversarially, and shows that for assignment matrices based on \textit{balanced incomplete block designs} (BIBD), an optimal decoding vector $w^*$ will always have fixed coefficients.

The most related works to ours are \cite{Charles} and \cite{limits}, which consider optimal decoding under random stragglers. The work \cite{Charles} was the first to use optimal decoding in the approximate gradient setting, and established that the the FRC-based assignment of \cite{Tandon} (which is also identical to that in \cite{ErasureHead}) achieves the decoding error $\frac{1}{N}\mathbb{E}|\alpha^* - \mathbbm{1}|_2^2 = p^d$ over random stragglers, which is optimal over all schemes with a replication factor of $d$. They show that the this FRC performs poorly in the adversarial setting, and so they also provide a random construction called a regularized Bernoulli Gradient Code (rBGC), which they suggest is harder to exploit by a computationally bounded adversary. In \cite{ErasureHead}, the authors provide bounds on the convergence rate of coded gradient descent using the FRC and optimal decoding under random stragglers. The work \cite{limits} provides upper and lower bounds on the computational load required to achieve a desired decoding error with high probability over random stragglers. Their upper bound is based on a construction using batch raptor codes (BRC) which achieves $\frac{1}{N}\mathbb{E}|\alpha^* - \mathbbm{1}|_2^2 = 1/e^{O(d)}$. We summarize the most relevant results from the work on approximate gradient codes in \Cref{comp_table}. 
To our knowledge, ours is the first analysis of an assignment scheme that achieves a decoding error decaying exponentially in $d$ for random straggler and a decdoing error less than $\frac{pm}{N}$ for adversarial stragglers. We show that the decoding error under random stragglers in our scheme is near-optimal as a function of the computational load, while the decoding error with adversarial stragglers is nearly twice as small as that of the FRC of \cite{Tandon}. 

Unlike many previous works that only study the decoding error $|\alpha^* - \mathbbm{1}|_2$ \cite{ExpanderCode, limits, bibd}, we also provide convergence results for both the random and adversarial settings. To the best of our knowledge, our work gives the first provable convergence results for approximate gradient coding with adversarial stragglers, although we note that there have been convergence results shown in other adversarial settings of gradient descent~\cite{haddock,byzantine}.

Other work such as \cite{LDPC} also considers the problem of approximate gradient coding, but differs from our framework in that their codes are not based on assignment matrices, or require specific types of loss functions.

\subsection{Organization}
In Section~\ref{sec:const}, we describe our construction of approximate gradient codes, and give some intuition for why we can show good bounds on the decoding error of our constructions. In \Cref{sec:characterization}, we characterize the optimal coefficents $w*$ and the resulting vector $\Astar = Aw^*$ in terms of the the straggling machines in a graph assignment scheme. In \Cref{ExpanderSec}, we prove our main result \Cref{expmainfull}, on the performance of graph-based assignment schemes in the setting of random stragglers. In \Cref{adversarial}, we prove \Cref{adversarial_good_expander}, on the robustness of graph-based assignment schemes to adversarial stragglers. In \Cref{convergence}, we state \Cref{conv} on the convergence of gradient descent for random stragglers. In \Cref{sec:adv_convergence}, we state Proposition~\ref{prop:formal_adv_conv} on the convergence of gradient descent for adversarial stragglers. In \Cref{sec_simulations}, we provide simulations which demonstrate our theoretical claims. We conclude in \Cref{conclusion}. Some proofs are deferred to the appendix.
\subsection{Notation}
We will use $|\cdot|_2$ to denote the $2$-norm of a vector or the operator norm of a matrix. For a graph $G = (V, E)$ and any sets of vertices $S, T \subset V$, we will denote by $E(S, T)$ the set of edges between vertices in $S$ and vertices in $T$. We will denote by $\partial(S)$ the edges $E(S, V \setminus S)$. For an edge $e \in E$, we denote by $\delta(e)$ the two endpoints of the edge $e$.

% We will used the term \textit{spectral expansion} of a graph to mean the difference between the largest and second-largest eigenvalue of the adjacency matrix. We will say a graph $G$ is a $(d, \lambda)$-expander if it is $d$-regular and has spectral expansion $\lambda$.

% We use the asymptotic notation big-$O$ and little-$o$ to denote limiting behaviour as $m \rightarrow \infty$. 

Let $\mathcal{S}_n$ denote the symmetric group on $n$ elements, and for a graph $G = (V, E)$ on $n$ vertices, let $\Aut(G) \subset \mathcal{S}_n$ denote the set of graph automorphisms of $G$. We say that a graph is $\textit{vertex transitive}$ if for any vertices $u, v \in V$, there exists some automorphism $\sigma \in \Aut(G)$ such that $\sigma(u) = v$. We denote the action of an automorphism $\sigma$ on a set $S \subset V$ in the following natural way: $\sigma(S) = \{\sigma(v) : v \in S\}.$ For a permutation $\rho \in \mathcal{S}_n$, we denote the action of $\rho$ on a vector $\beta$ in the following way: $\rho(\beta)_i = \beta_{\rho(i)}$.

\section{Our Construction}\label{sec:const}
In our construction, each machine holds exactly two data blocks, each comprised of $\frac{dN}{2m}$ data points. We introduce the parameter $n := \frac{2m}{d}$ to denote the number of data blocks. We summarize these parameters in \Cref{NotationTable}.

\begin{table}[t]
\begin{center}
\begin{tabular}{| c | c | }
  \hline 
  $m$    & Number of machines  \\
  \hline
  $N$    & Number of data points  \\
  \hline
  $n$ & Number of data blocks in a graph-based scheme\\
  \hline
  $\ell$ & Computational load (maximum points per machine)\\
  \hline
  $d$ & Replication Factor (averaged over all data points)\\
  \hline 
\end{tabular}
\caption{Parameters in this work.  We always have $d = 2m/n = m \ell / N$.  A useful parameter regime to keep in mind is the setting where $N = m$ and $\ell = d$.}\label{NotationTable}
\end{center}
\end{table}

\begin{remark}
Because our assignment schemes are regular---that is, each data block is replicated an equal number of times---each data point will be replicated exactly $d$ times. Observe that the \textit{computational load} $\ell$, the maximum number of data points per machine, equals $\frac{dN}{m}$. As the regime $m = N$ is the most commonly studied, it is convenient to think of $d$ as equal to $\ell$ when comparing our results to other work, some of which state results in terms of $\ell$. In general, when $m = N$, we must have $\ell \geq d$.
\end{remark}
We can describe these assignment schemes using a graph on $n$ vertices with $m$ edges. We abuse notation and use the assignment matrix $A$ to denote the $n \times m$ assignment matrix of blocks to machines, 
rather than the $N \times m$ assignment matrix of points to machines.   Thus, all of our results are in terms of the replication factor $d = \frac{2m}{n}$, which is independent of the block size. 
\begin{definition}\label{graphmatrix}
A graph assignment scheme corresponding to a graph $G$ with $n$ vertices and $m$ edges is a matrix $A \in \{0,1\}^{n \times m}$ in which $A_{ij} = 1$ if the $j$th edge of $G$ has $i$ as an endpoint.
\end{definition}
An example of \Cref{graphmatrix} is shown in Figure~\ref{fig:asst_graph_example}.
\begin{remark}
In contrast to other works (such as \cite{ExpanderCode}), which have also designed codes based on graphs, the graph we consider is \em not \em a bipartite graph where left vertices correspond to data blocks and right vertices correspond to machines.  Rather, it is the non-bipartite graph where the data blocks are the vertices and the machines are the edges.
\end{remark}

Recall that to minimize the decoding error, we want to show that $\Astar$ is close to $\mathbbm{1}$, such that the gradient updates given in \Cref{star} are as close as possible to those in batch gradient descent. By thinking of an assignment scheme as a graph $G$ as above, we are able to characterize $\Astar$ in terms of the connected components of a random sparsification of $G$. In \Cref{sec:characterization}, we show that $\Astar_i$ will be close to $1$ if vertex $i$ is in component which is either non-bipartite, or bipartite with close to balanced sides.

Expanders are good examples of sparse graphs which have large non-bipartite components under random sparsification. We show this in \Cref{bipartite} by proving that randomly sparsified expanders have a giant connected component with high probability (\Cref{giant_expander}). To additionally guarantee that our gradient descent converges to the true minimum, we use vertex transitive expanders, such as Cayley graph expanders, which guarantee that $\mathbb{E}[\Astar] = c\mathbbm{1}$.

\section{Characterization of $\Astar$}\label{sec:characterization}

In this section, we characterize $\Astar$ in terms of the straggling machines in a graph assignment scheme. This will allow us to prove the desired properties of $\Astar$ by studying randomly sparsified graphs.

Suppose we have some graph assignment scheme $A$ corresponding to a graph $G$.
Recall that $\Astar = Aw^*,$ where
$$w^* \in \argmin_{w: w_j = 0 \text{ if machine } j \text{ straggles }}{|\mathbbm{1} - Aw|_2}.$$ We define $G(p)$ to be the random graph where each edge of $G$ is deleted with probability $p$.

We can think of $w^*$ as a weight vector which has one (possibly) non-zero coordinate $w^*_e$ for each edge $e$ in $G(p)$.  We can think of $\Astar$ as a vector where each coordinate $\Astar_v$ is the sum of weights $w^*_e$ of each edge $e$ incident to $v$.  See \Cref{fig2} for some examples of these.

It follows from \Cref{eq:optimal} that 
for any edge $e = (u, v)$, $\Astar$ satisfies  \begin{equation}\label{sum2}\Astar_u + \Astar_v = 2.\end{equation}
Indeed, at the optimum we have $0 = A^T(\mathbbm{1} - Aw^*) = A^T(\mathbbm{1} - \alpha^*)$, which implies that for all edges $e = (u,v)$ (which index the rows of $A^T$), we have $(1 - \alpha^*_u) + (1 - \alpha^*_v) = 0$, yielding \Cref{sum2}.

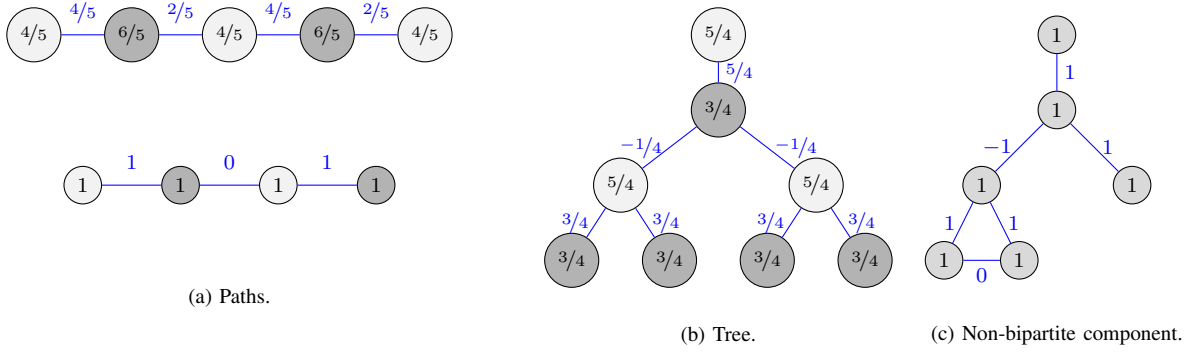
\begin{figure}[t]
  \begin{center}
\begin{tikzpicture}
\footnotesize
\tikzstyle{vertex}=[draw,circle,fill=black!15!white]
\tikzstyle{v0}=[draw,circle,fill=black!5!white]
\tikzstyle{v1}=[draw,circle,fill=black!30!white]

\begin{scope}[xscale=1.3]
\begin{scope}
\node[v0](v-2) at (-2,0) {$\nf{4}{5}$};
\node[v1](v-1) at (-1,0) {$\nf{6}{5}$};
\node[v0](v0) at (0,0) {$\nf{4}{5}$};
\node[v1](v1) at (1,0) {$\nf{6}{5}$};
\node[v0](v2) at (2,0) {$\nf{4}{5}$};
\draw[blue] (v-2)--(v-1)--(v0)--(v1)--(v2);
\node[blue] at (-1.5,.3) {$\nf{4}{5}$};
\node[blue] at (-.5,.3) {$\nf{2}{5}$};
\node[blue] at (.5,.3) {$\nf{4}{5}$};
\node[blue] at (1.5,.3) {$\nf{2}{5}$};
\end{scope}
\begin{scope}[yshift=-2cm]
\node[v0](v-2) at (-1.5,0) {$1$};
\node[v1](v-1) at (-.5,0) {$1$};
\node[v0](v0) at (.5,0) {$1$};
\node[v1](v1) at (1.5,0) {$1$};
\draw[blue] (v-2)--(v-1)--(v0)--(v1);
\node[blue] at (-1,.3) {$1$};
\node[blue] at (0,.3) {$0$};
\node[blue] at (1,.3) {$1$};

\node at (0,-1.5) {(a) Paths.};
\end{scope}
\begin{scope}[xshift=5cm,yshift=0cm, yscale=1]
\node[v0] (a0) at (0,0) {$\nf{5}{4}$};
\node[v1] (a1) at (0,-1) {$\nf{3}{4}$};
\node[v0] (a2) at (-1,-2) {$\nf{5}{4}$};
\node[v0] (a3) at (1,-2) {$\nf{5}{4}$};
\node[v1] (a4) at (-1.5,-3) {$\nf{3}{4}$};
\node[v1] (a5) at (-.5,-3) {$\nf{3}{4}$};
\node[v1] (a6) at (.5,-3) {$\nf{3}{4}$};
\node[v1] (a7) at (1.5,-3) {$\nf{3}{4}$};
\draw[blue] (a0) to node[pos=0.5,right] {$\nf{5}{4}$} (a1);
\draw[blue] (a1) to node[pos=0.5,left] {$\nf{-1}{4}$} (a2);
\draw[blue] (a1) to node[pos=0.5,right] {$\nf{-1}{4}$} (a3);
\draw[blue] (a2) to node[pos=0.5,left] {$\nf{3}{4}$} (a4);
\draw[blue] (a2) to node[pos=0.5,right] {$\nf{3}{4}$} (a5);
\draw[blue] (a3) to node[pos=0.5,left] {$\nf{3}{4}$} (a6);
\draw[blue] (a3) to node[pos=0.5,right] {$\nicefrac{3}{4}$} (a7);
\node at (0,-4) {(b) Tree.};
\end{scope}
\end{scope}
\begin{scope}[xshift=11cm,yshift=0cm, yscale=1]
\node[vertex] (a0) at (0,0) {$1$};
\node[vertex] (a1) at (0,-1) {$1$};
\node[vertex] (a2) at (-1,-2) {$1$};
\node[vertex] (a3) at (1,-2) {$1$};
\node[vertex] (a4) at (-1.5,-3) {$1$};
\node[vertex] (a5) at (-.5,-3) {$1$};
\draw[blue] (a0) to node[pos=0.5,right] {$1$} (a1);
\draw[blue] (a1) to node[pos=0.5,left] {$-1$} (a2);
\draw[blue] (a1) to node[pos=0.5,right] {$1$} (a3);
\draw[blue] (a2) to node[pos=0.5,left] {$1$} (a4);
\draw[blue] (a2) to node[pos=0.5,right] {$1$} (a5);
\draw[blue] (a4) to node[pos=0.5,below] {$0$} (a5);
\node at (0,-4) {(c) Non-bipartite component.};
\end{scope}
\end{tikzpicture}
\end{center}
  \caption{The optimal choice of $w^*$ (edge labels) and $\alpha^*$ (vertex labels) in various connected components. } %(a) Paths. (b) Tree. (c) A non-bipartite component.}
  \label{fig2}
\end{figure}

We can make the following observations which follow from \Cref{sum2}:
\begin{enumerate}
    \item For any set of vertices $v$ in a single connected component, $|1 - \Astar_v|$ is the same.  Indeed, for an edge $(u,v)$, \Cref{sum2} implies that $1 - \Astar_u = \Astar_v - 1$, and this relationship extends to a whole connected component.
    \item If a component contains an odd cycle of vertices (i.e., is not bipartite), then $\Astar_v = 1$ for all of the vertices in the component. Indeed, as above, the sign of $1 - \Astar_v$ alternates along edges of the component, which would produce a contradiction if $1 - \Astar_v$ was not $0$ at every vertex in the odd cycle.
    \item If a component $C = L \cup R$ is bipartite with $|L| \geq |R|$, then $\Astar_u = 1 + \frac{|L| - |R|}{|L| + |R|}$ if $u \in R$ and $\Astar_v = 1 - \frac{|L| - |R|}{|L| + |R|}$ if $v \in L$. This is true because the sum of all edge weights going into vertices in $L$ is equal to the sum of all edge weights going into vertices in $R$, so $\sum_{u \in R}{\Astar_u} = \sum_{v \in L}{\Astar_v}$. Using items (1) and (2) to conclude that $\alpha_u^*$ is constant on $u \in R$ and $\alpha_v^* = 2 - \alpha_u^*$ is constant on $v \in L$ yields the statement.
\end{enumerate}
These observations suggest the approach that we will use in \Cref{ExpanderSec}: we can bound the contribution to $|\Astar - \mathbbm{1}|_2^2$ of a particular connected component by simply knowing whether that component is bipartite. %and, if so, how balanced it is. 

Algorithmically, given the set of non-straggling machines, the observations above allow the parameter server to compute the optimal coefficients $w^*$ in linear time in $m$. First the parameter server performs a breadth-first search on $G(p)$ to divide the graph into connected components, and determines the two sides $L$ and $R$ of any bipartite components. For each connected component, the parameter server can then compute $\Astar_v$ for each $v$ in the component. Finally, the parameter server performs a depth-first search on each component to label each edge with the a value $w^*_e$ such that the sum of all edges incident to vertex $v$ equal $\Astar_v$. Note that the value of $w^*_e$ may depend on the order edges are discovered in the depth-first search, but the vector $\Astar$ is unique.

\section{The Decoding Error $|\alpha^* - \mathbbm{1}|_2$ under Random Stragglers}\label{ExpanderSec}
In this section we prove our main result about expander graph assignments under random stragglers.

\begin{theorem}\label{expmainfull}
Let $G = (V, E)$ be any vertex transitive graph with $n$ vertices, $m$ edges, and spectral expansion $\lambda$. Let $A$ be the assignment matrix given by $G$, in accordance with \Cref{graphmatrix}. Suppose for some positive~$\epsilon$,
\begin{enumerate}
\item $\lambda > 1.5$;
\item $(\frac{\lambda}{2} + 1)(1 - pe^{1/\lambda}) \geq 1 + \epsilon$;
\item $\left(1 - \frac{e^2p^{\lambda\left(1 - \frac{1}{3 + \epsilon}\right)}}{(1 - pe^{1/\lambda})^2}\right)\left(\lambda - \left(2d - \lambda\right)\frac{e^2p^{\lambda\left(1 - \frac{1}{3 + \epsilon}\right)}}{(1 - pe^{1/\lambda})^2}\right) \geq \epsilon$;
\item $\frac{n}{\log(n)^2} \geq \frac{4(1 + \epsilon)}{\epsilon^2}\left(2 - 2\log(\epsilon) + 2\log(1 + \epsilon) - \log\left(1 - pe^{\frac{1}{\lambda}}\right)\right)$.
\end{enumerate}

If each machine straggles independently with probability $p$, then
\begin{enumerate}
    \item $\mathbb{E}[\Astar] = r\mathbbm{1}$ for some $r \geq 1 - \frac{6}{n} - t$;
    \item For all $i$, $\mathbb{E}[(\Astar_i - r)^2] \leq \frac{6}{n} + t$;
    \item $|\mathbb{E}[(\Astar - r\mathbbm{1})(\Astar - r\mathbbm{1})^T]|_2 \leq 2k^2\left(t + \frac{6}{n}\right)^2 + 24$,
\end{enumerate}
where $t = \frac{e^2p^{\lambda\left(1 - \frac{1}{3 + \epsilon}\right)}}{(1 - pe^{1/\lambda})^2},$ $$k = \frac{2(1 + \epsilon)}{\epsilon^2}\left(2\log(n) - 2\log(\epsilon) + 2\log(1 + \epsilon) - \log(1 - p)\right),$$ and all expectations are over the random stragglers.
\end{theorem}

\begin{remark}
The expected error $\frac{1}{n}\mathbb{E}\left[|\Abarstar - \mathbbm{1}|_2^2\right]$ is lower-bounded by $p^{d}$, because this is the probability that a fixed data block is stored only at straggling machines. 
\Cref{expmainfull} implies that, for good expanders, the variance of $\Astar$ shrinks exponentially in the replication factor, $d$. One example of such graphs are the Lubotzky-Phillips-Sarnak (LPS) construction \cite{LPS} of Ramanujan Cayley graphs, where $\lambda \geq d - 2\sqrt{d - 1}$. In this sense, our result is tight in $d$ up to factors of $p^{o(d)}$.
\end{remark}

\Cref{expmainfull} is a consequence of \Cref{giant_expander} and \Cref{bipartite} below.  We will state these results and then prove \Cref{expmainfull} assuming them.  After proving \Cref{expmainfull}, we will prove \Cref{giant_expander} and \Cref{bipartite} at the end of this section.

\begin{theorem}\label{giant_expander}
Let $G$ be any $d$-regular $\lambda$-spectral expander with $n$ vertices and suppose for some positive~$\epsilon$,

\begin{enumerate}
\item $\lambda > 1.5$;
\item $(\frac{\lambda}{2} + 1)(1 - p) \geq 1 + \epsilon$;
\item $\frac{n}{\log(n)^2} \geq \frac{2(3 + \epsilon)(1 + \epsilon)}{\epsilon^2}\left(2 - 2\log(\epsilon) + 2\log(1 + \epsilon) - \log(1 - p)\right)$.
\end{enumerate}

Let $G(p)$ be a random sparsification of $G$, where each edge of $G$ is deleted randomly with probability $p$.
With probability at least $1 - \frac{5}{n}$,
\begin{enumerate}
    \item $G(p)$ has a giant component of size at least $n\left(1 - \frac{ep^{\lambda\left(1 - \frac{1}{3 + \epsilon}\right)}}{(1 - p)^2}\right)$ vertices;
    \item Every vertex is either in a component of size at most $k$, where $k$ is as in \Cref{expmainfull}, or is in a component of size greater than $n/2$.
\end{enumerate} 
\end{theorem}
\begin{corollary}\label{bipartite}
Let $G$ be any $d$-regular $\lambda$-spectral expander with, and suppose for some positive $\epsilon$,
\begin{enumerate}
\item $\lambda > 1.5$;
\item $(\frac{\lambda}{2} + 1)(1 - pe^{1/\lambda}) \geq 1 + \epsilon$;
\item $\left(1 - \frac{e^2p^{\lambda\left(1 - \frac{1}{3 + \epsilon}\right)}}{(1 - pe^{1/\lambda})^2}\right)\left(\lambda - \left(2d - \lambda\right)\frac{e^2p^{\lambda\left(1 - \frac{1}{3 + \epsilon}\right)}}{(1 - pe^{1/\lambda})^2}\right) \geq \epsilon$;
\item $\frac{n}{\log(n)^2} \geq \frac{4(1 + \epsilon)}{\epsilon^2}\left(2 - 2\log(\epsilon) + 2\log(1 + \epsilon) - \log\left(1 - pe^{\frac{1}{\lambda}}\right)\right)$.
\end{enumerate}

Let $G(p)$ a random sparsification of $G$, where each edge of $G$ is deleted randomly with probability $p$. Then with probability at least $1 - \frac{6}{n}$, 
\begin{enumerate}
    \item $G(p)$ has a non-bipartite giant component of size at least $n\left(1 - \frac{e^2p^{\lambda\left(1 - \frac{1}{3 + \epsilon}\right)}}{(1 - pe^{1/\lambda})^2}\right)$.
    \item  Every vertex is either in a component of size at most $k$, where $k$ is as in \Cref{expmainfull}, or is in a component of size greater than $n/2$.
\end{enumerate}
\end{corollary}

We begin by proving \Cref{expmainfull} assuming \Cref{giant_expander} and \Cref{bipartite}.
\begin{proof}(\Cref{expmainfull})
Recall that because the graph $G$ is vertex-transitive, the distribution of $\Astar_i$ is equivalent for every vertex $i$, and so $\mathbb{E}[\Astar]$ is some multiple of $\mathbbm{1}$. Throughout we will use the fact that by \Cref{bipartite}, with probability at least $1 - \frac{6}{n}$, at least $(1 - t)n$ vertices $i$ are in a non-bipartite components, and hence have $\Astar_i = 1$.  Here, $t$ is as in the statement of \Cref{expmainfull}.

For the first statement, for any $i$, because $\Astar_i \geq 0$, we have $$\mathbb{E}[\Astar_i] \geq \Pr[\Astar_i = 1] \geq \Pr[\text{vertex $i$ is in a non-bipartite component}] \geq 1 - \frac{6}{n} - t.$$

For the second statement, $$\mathbb{E}[(\Astar_i - \mathbb{E}[\Astar_i])^2] \leq \mathbb{E}[(\Astar_i - 1)^2] \leq 1 - \Pr[\Astar_i = 1] \leq \frac{6}{n} + t,$$
where we used the fact that $|1 - \Astar_i| \leq 1$ always.

The third statement follows from the following lemma, which we prove in Appendix~\ref{apx:cov}:
\begin{restatable}{lemma}{covlemma}\label{exp_cov_lemma}
Given the assumptions in \Cref{expmainfull}, we have
$$|\mathbb{E}[(\Astar - r\mathbbm{1})(\Astar - r\mathbbm{1})^T]|_2 \leq 2k^2\left(t + \frac{6}{n}\right)^2 + 24,$$
where $r, t$ and $k$ are as defined in \Cref{expmainfull}.
\end{restatable}
\end{proof}

We will use the well-known expander mixing lemma (see, e.g., \cite{HLW06}) in the proof of \Cref{giant_expander}.
\begin{lemma}(Expander Mixing Lemma)
For any sets $S$ and $T$ in a $d$ regular graph with expansion $\lambda$, we have 
$|E(S, T)| \geq d\frac{|S||T|}{n} - (d - \lambda)\sqrt{|S||T|(1 - |S|/n)(1 - |T|/n)}$.
\end{lemma}
% \begin{corollary}
% If $|T| = (1 - q - \epsilon)n$ and $|S| = qn$, then $$E(S, T) \geq (1 - q - \epsilon)dqn - \lambda n\sqrt{q(1 - q - \epsilon)(1 - q)(q + \epsilon)} \geq (1 - q - \epsilon)qn\left(d - \lambda\sqrt{\frac{(1 - q)(q + \epsilon)}{(1 - q - \epsilon)q}}\right)$$
% \end{corollary}

Now, we are ready to prove \Cref{giant_expander} and \Cref{bipartite}, which will make up the rest of this section.
\begin{proof}(\Cref{giant_expander})
We consider a slight generalization of the growing process on $G$ used in \cite{gt} which begins from a single edge $e$ and discovers the connected component of $e$ in $G(p)$. While this growing process considers all of $G$, it will ``discard" edges it discovers which are not in $G(p)$. We modify this growing process to have a set of vertices $V_0$ as input and discover all vertices of $G(p)$ in the connected components of all $v \in V_0$. Let $C_u$ denote the connected component containing $u$ in $G(p)$. The growing process algorithm is described in Algorithm~\ref{gp}.
\begin{algorithm}[h]
 \caption{Growing Process ($V_0$)}\label{gp}
\begin{algorithmic}
\Procedure{Growing Process}{$V_0$}
%\hspace*{\algorithmicindent} \textbf{Input:} $V_0$
\State \textbf{Result:} $\bigcup_{v \in V_0}{C_v}$
\State $t \leftarrow 0$
\State $S_0, B_0 \leftarrow \emptyset$\;
\State $F_0 \leftarrow \partial(V_0)$ \Comment{$F_t$ is the frontier of unexplored edges}
\While{$|F_t| > 0$}
\State Choose $e \in F_t$ arbitrarily
\State   $t \leftarrow t + 1$
\State $X_t \sim \text{Bernoulli}(1 - p)$
\If{$X_t = 1$}
\State $S_t \leftarrow S_{t - 1} \cup e $ \Comment{$S_t$ is the set of explored edges in $G(p)$}
\State $V_t \leftarrow V_{t - 1} \cup \delta(e)$ \Comment{$V_t$ is the set of explored vertices in $\bigcup_{v \in V_0}{C_v}$}
\State $B_t \leftarrow B_{t - 1}$  \Comment{$B_t$ is the set of explored edges in $G$ but not $G(p)$}
\Else
\State $B_t \leftarrow B_{t - 1} \cup e$\;
\State $S_t \leftarrow S_{t - 1}$\;
\State $V_t \leftarrow V_{t - 1}$\;
\EndIf
\State   $F_t \leftarrow \partial(V_t) \setminus B_t$
\EndWhile \\
\Return{$V_t$}
\EndProcedure
\end{algorithmic}
\end{algorithm}

Because each edge can only be chosen once, this growing process is stochastically equivalent to revealing edges from $G(p)$ and setting $X_t$ equal to $1$ if edge $e$ is in $G(p)$ and otherwise $0$. 
Notice that $S_t$, the set of explored edges which lie in $G(p)$, is always a forest, so $|V_t| = |V_0| + |S_t|$. Further $|B_t| = t - |S_t|$.

In the following claim, we lower bound the size of the frontier of unexplored edges, $F_t$.
\begin{claim}\label{frontier}
At the end of any step $t$ where $|V_t| \leq n/2$, we have
$$|F_t| \geq \sum_{i = 1}^t{\left(X_i\left(\alpha_t + 1\right) - 1\right)} + |V_0|\alpha_t,$$

where $\alpha_t = \lambda\max\left(\frac{1}{2}, 1 - \frac{t + |V_0|}{n}\right)$.
\end{claim}
\begin{proof}
By the expander mixing lemma, we have
$$|\partial(V_t)| \geq \lambda|V_t|\left(1 - \frac{|V_t|}{n}\right) \geq \alpha_t|V_t|.$$

Hence \begin{equation}
\begin{split}
|F_t| &= |\partial(V_t) \setminus B_t| \\
& \geq \alpha_t|V_t| - |B_t|\\
&= \alpha_t(|S_t| + |V_0|) - (t - |S_t|)\\
&= \left(\alpha_t + 1\right)|S_t| + |V_0|\alpha_t - t\\
&= \sum_{i = 1}^t{\left(X_i\left(\alpha_t + 1\right) - 1\right)} + |V_0|\alpha_t.
\end{split}
\end{equation}
\end{proof}

Let \begin{equation}\label{k_def}k = \frac{2(1 + \epsilon)}{\epsilon^2}\left(2\log(n) - 2\log(\epsilon) + 2\log(1 + \epsilon) - \log(1 - p)\right).\end{equation} By the third assumption in the theorem,
\begin{equation}\label{kmin}
k \leq \frac{n}{(3 + \epsilon)\log(n)}.
\end{equation}
In the following claim, we will only use the fact that $\alpha_t \leq \frac{1}{2}$ for the values $t$ we consider.

We begin with the following claim about the growing process.
\begin{claim}\label{giantlemma}
With probability at least $1 - \frac{4}{n}$, simultaneously for all vertices $v$ either:
\begin{enumerate}
    \item The size of the component $v$ lies in is less than $k$.
    \item The size of the component $v$ lies in strictly greater than $n/2$. 
\end{enumerate}
\end{claim}
\begin{proof}
Conditioned on $v$ being in a component of size at least $k$, there must be at least $k(1 + \alpha_{k})$ steps in the growing process on input $\{v\}$.

Hence 
\begin{align*}
\Pr\left[|C(v)| \leq n/2 \bigm| |C(v)| \geq k\right] &\leq \sum_{t = k(1 + \alpha_{k})}^{\min t - 1: |V_t| > n/2}{\Pr\left[|F_t| = 0 \bigm| |F_{t - 1}| > 0\right]} \\ 
&\leq \sum_{t = k(1 + \alpha_{k})}^{\infty}{ \Pr\left[ \sum_{i = 1}^t{\left(X_i(\alpha_t + 1) - 1\right)} \leq - \alpha_t \right]}.
\end{align*}

By the Chernoff bound for Bernoulli random variables,
\begin{align*}
\Pr\left[ \sum_{i = 1}^t{\left(X_i(\alpha_t + 1) - 1\right)} \leq - \alpha_t \right] 
&\leq \Pr\left[ \sum_{i = 1}^t{X_i} \leq \frac{t}{\alpha_t + 1} \right] \\
&=   \Pr\left[ \sum_{i = 1}^t{X_i} \leq (1 - p)t\frac{1}{(\alpha_t + 1)(1 - p)} \right]\\
&\leq \exp\left(-\frac{(1 - p)t}{2}\left(1 - \frac{1}{(\alpha_t + 1)(1 - p)}\right)^2\right)
\end{align*}

Hence by the second assumption in the theorem and the definition of $k$,
\begin{align*}
\sum_{t = k(1 + \alpha_{k})}^{\infty}{ \Pr\left[ \sum_{i = 1}^t{\left(X_i(\alpha_t + 1) - 1\right)} \leq - \alpha_t \right]} 
&\leq \sum_{t = k(1 + \alpha_{k})}^{\infty}{\exp\left(-\frac{(1 - p)t}{2}\left(1 - \frac{1}{1 + \epsilon}\right)^2\right)}\\
&=  \frac{\exp\left(-\frac{(1 - p)(1 + \alpha_{k})k}{2}\left(1 - \frac{1}{1 + \epsilon}\right)^2\right)}{1 - \exp\left(-\frac{(1 - p)}{2}\left(1 - \frac{1}{1 + \epsilon}\right)^2\right)} \\
&\leq \frac{\exp\left(-\frac{k\epsilon^2}{2(1 + \epsilon)}\right)}{\frac{(1 - p)}{4}(1 - \frac{1}{1 + \epsilon})^2} = \frac{4}{n^2}.
\end{align*}

Taking a union bound over all $n$ vertices yields \Cref{giantlemma}.
\end{proof}
Since there can be at most one connected component of size at least $n/2 + 1$,  with probability at least $1 - \frac{4}{n}$, all vertices in components of size at least $k$ are in the same giant component.

It remains to bound the number of vertices in small components with high probability. Let $I_v$ be the indicator random variable of the event $|C_v| \leq k$, and let $Y = \sum_v{I_v}$.

\begin{claim}\label{concentration}
$$\Pr\left[Y \geq en\frac{p^{\lambda\left(1 - \frac{1}{3 + \epsilon}\right)}}{(1 - p)^2}\right] \leq \frac{1}{n}.$$
\end{claim}

It follows from this claim and \Cref{giantlemma} that with probability at least $1 - \frac{4}{n} - \frac{1}{n}$, there is a giant component of size at least $n\left(1 - \frac{ep^{\lambda\left(1 - \frac{1}{3 + \epsilon}\right)}}{(1 - p)^2}\right),$ and no vertices are in components of size greater than $k$ but at most $n/2$. This will establish \Cref{giant_expander}.

The remainder of this proof is devoted to proving \Cref{concentration}. We will use \Cref{walk} and \Cref{moments} below to help us bound the moments of $Y$ and prove \Cref{concentration}.

\begin{lemma}\label{walk}
Let $S(\alpha) = \sum_{i}^{\infty}{X_i}$ be a random walk where $X_i = \alpha$ with probability $1 - p$ and $-1$ with probability $p$. Let $S^{(\beta)}(\alpha) = \beta + S(\alpha)$. If $\alpha > 1$, for any positive $c$, the probability that $S^{(c\alpha)}(\alpha)$ goes below zero is at most $\left(\frac{p^\alpha}{(1 - p)^2}\right)^{c}$.
\end{lemma}
\begin{proof}
For any $\beta$, let $d_{\beta}(\alpha)$ be the probability of extinction of $S^{(\beta)}(\alpha)$. We upper bound $d_{c\alpha}(\alpha)$ with $d_{\alpha}(\lfloor\alpha\rfloor)$. For the random walk $S(\lfloor\alpha\rfloor)$, we have $d_\beta(\lfloor\alpha\rfloor) = d_{\lceil{\beta}\rceil}(\lfloor\alpha\rfloor) = d_1(\lfloor\alpha\rfloor)^{\lceil{\beta}\rceil}$ and $d_1(\lfloor\alpha\rfloor) = p + (1 - p)d_1(\lfloor\alpha\rfloor)^{\lfloor\alpha\rfloor + 1},$ which yields $d_1(\lfloor\alpha\rfloor) \leq \frac{p}{(1 - p)^{\frac{1}{\lfloor\alpha\rfloor}}}.$

Hence for any positive integer $c$, we have $$d_{c\alpha}(\alpha) \leq d_{c\alpha}(\lfloor{\alpha}\rfloor) \leq \left(\frac{p}{(1 - p)^{\frac{1}{\lfloor{\alpha}\rfloor}}}\right)^{\lceil{c\alpha}\rceil} \leq \left(\frac{p^\alpha}{(1 - p)^{\frac{\alpha}{\lfloor{\alpha}\rfloor}}}\right)^{c} \leq \left(\frac{p^\alpha}{(1 - p)^2}\right)^{c}.$$
\end{proof}
% This completes the proof of \Cref{walk}.

\begin{claim}\label{moments}
For any set $S$ of $c \leq \log(n)$ vertices, the probability that every vertex in $S$ is in a component of size at most $k$ is at most $\frac{\left(p^{\lambda\left(1 - \frac{1}{3 + \epsilon}\right)}\right)^c}{(1 - p)^2}.$
\end{claim}
\begin{proof}
We know that $$\Pr\left[\max_{v \in S}{|C_v| \leq k}\right] \leq \Pr\left[\left|\bigcup_{v \in S} C_v\right| \leq ck\right].$$

This probability on the right hand side is the probability that our growing process starting with $V_0 = S$ terminates before reaching a size of $ck$. From \Cref{kmin}, we have that $ck \leq n\left(1 - \frac{1}{3 + \epsilon}\right)$ for $c \leq \log(n)$.
Hence the probability that the growing process terminates before reaching a size of $ck$ is upper bounded by the probability that the random walk $$\sum_{i = 1}^{\infty}{\left(X_i\left(\lambda\left(1 - \frac{1}{3 + \epsilon}\right) + 1\right) - 1\right)} + c\lambda\left(1 - \frac{1}{3 + \epsilon}\right)$$ becomes extinct.
Observing that $\lambda\left(1 - \frac{1}{3 + \epsilon}\right) > 1$ by the first assumption in the theorem, \Cref{walk} yields the desired bound.  This completes the proof of \Cref{moments}.
\end{proof}

The probability bound in \Cref{moments} implies that the first $\log(n)$ moments of the random variable $Y$ are upper bounded by the moments of $\text{Binonimal}(n, q)$, where $q = \frac{p^{\lambda\left(1 - \frac{1}{3 + \epsilon}\right)}}{(1 - p)^2}$. We can bound this moment using the following proposition, which we prove in Appendix~\ref{apx:binom} for completeness. 
\begin{restatable}{proposition}{binomfact}\label{prop:binomialfact}
For any $n, q$ and $c \leq \log(n)$, $$\mathbb{E}\left[\left(\text{Binomial}(n, q) - nq\right)^c\right] \leq \left(2q\sqrt{nc}\right)^c.$$ 
\end{restatable}
Given \Cref{prop:binomialfact} and 
applying Markov's inequality to the $\log(n)$th centralized moment of $Y$ yields 
$$\Pr[Y \geq enq] = \Pr[(Y - \mathbb{E}[Y])^{\log(n)} \geq \left(neq - nq\right)^{\log(n)}] \leq \frac{\left(2q\sqrt{n\log(n)}\right)^{\log(n)}}{\left(nq(e - 1)\right)^{\log(n)}} \leq \left(\frac{1}{e}\right)^{\log(n)} = \frac{1}{n}.$$
This establishes \Cref{concentration}, and hence \Cref{giant_expander}.
\end{proof}

Next, we prove \Cref{bipartite}, which follows from \Cref{giant_expander}.
\begin{proof}(\Cref{bipartite})
We will prove this via ``edge sprinkling", the process described below. Let $q = pe^{1/\lambda}$.

Consider the following random process to create $G(p)$, which is equivalent to deleting edges with probability $p$:
\begin{itemize}
    \item Step 1: Create the graph $G(q)$ from $G$ by deleting edges with probability $q$.
    \item Step 2: For every edge of $G$ where there is not an edge in $G(q)$, add an edge with probability $1 - \frac{p}{q}$.
\end{itemize}

By \Cref{giant_expander}, with probability at least $1 - \frac{5}{n}$, after the first step, the graph $G(q)$ has a giant component $C$ of size $n\left(1 - s\right)$, where $s = \frac{e^2p^{\lambda\left(1 - \frac{1}{3 + \epsilon}\right)}}{(1 - pe^{1/\lambda})^2}.$

If $C$ is bipartite, there is only one way to choose the left and right sides of the graph $L$ and $R$ with $|R| \geq |L|$. Let $r = |R|/n$, so $r \geq \frac{1 - s}{2}$.
By the expander mixing lemma, there are at least $n(dr^2 - (d - \lambda) r(1 - r))$ edges in $G$ inside $R$. 

By the third assumption in \Cref{bipartite}, we have $$n(dr^2 - (d - \lambda)r(1 - r)) = nr(\lambda(1 - r) -d(1 - 2r)) \geq \frac{1}{2}nr(\lambda - (d + \lambda)s) \geq \frac{n\epsilon}{4}.$$ 

During the second step, the probability that no edge is added inside $R$ is at most
$$\left(\frac{p}{q}\right)^{\frac{n\epsilon}{4}} \leq \exp\left(-\frac{n\epsilon}{4\lambda}\right) \leq \frac{1}{n}.$$

It follows that with probability at least $1 - \frac{5}{n} - \frac{1}{n}$, there is a giant non-bipartite component. The second statement of \Cref{bipartite} follows from the second statement in \Cref{giant_expander}.  This completes the proof of \Cref{bipartite}.
\end{proof}

\section{The Decoding Error $|\alpha^* - \mathbbm{1}|_2$ under Adversarial Stragglers}\label{adversarial}
In this section we show how the spectral properties of an assignment matrix $A$ can be leveraged to bound the adversarial error $|\alpha^* - \mathbbm{1}|_2^2$. We show that graph-based assignment schemes which use graphs with large expansion perform nearly twice as well as the FRC of \cite{Tandon} in the adversarial setting.

The following proposition and its proof are almost the same as Proposition 29 in \cite{ExpanderCode}.

\begin{proposition}\label{prop:generic_adv}
Let $A \in \mathbb{R}^{N \times m}$ be any assignment matrix on $N$ data points, for which each data point is replicated exactly $d$ times, and each machine holds exactly $\ell$ data points. Let $\sigma_2$ be the second largest singular value of $A$. For any set of $S$ of $s$ stragglers, there exist some decoding coefficients $w = w(S)$ such that 
$$\frac{1}{N}|\alpha - \mathbbm{1}|_2^2 \leq \frac{1}{N}\left(\frac{\sigma_2}{d}\right)^2\frac{sm}{m - s}$$.
\end{proposition}
\begin{proof}
For any set $S$ of $s$ stragglers, let $w_i = \frac{m}{d(m - s)}$ for $i \notin S$ and $w_i = 0$ for $i \in S$. Then 
\begin{equation}|\alpha - \mathbbm{1}|_2^2 = \left|Aw - \frac{1}{d}A\mathbbm{1}\right|_2^2 = \left|Az
\right|_2^2,
\end{equation}
where $z = w - \frac{1}{d}\mathbbm{1}$. 
We observe that $A$ has top singular value $\sigma_1 = \sqrt{\ell d}$ and top right singular vector $\mathbbm{1}/\sqrt{m}$; this follows from the fact that $A^TA$ evidently has top eigenvector $\mathbbm{1}/\sqrt{m}$ and top eigenvalue $\ell d$.  

%We can show that $\frac{1}{\sqrt{m}}\mathbbm{1}$ is the top right singular value of $A$ by considering the top eigenvalue of $A^TA$. Indeed, we have $A^TA\mathbbm{1} = \ell d\mathbbm{1}$, and since all entries of $A^TA$ are positive, for any $x \in \mathbbm{R}^d,$ we have $(A^TAx)_{j} < \ell d x_j$, where $j = \argmax_i{x_i}$.

Observe that $z \perp \mathbbm{1}$ and $|z|_2^2 = \frac{1}{d^2}\left(s + (m - s)\frac{s^2}{(m - s)^2}\right) = \frac{ms}{(m - s)d^2}$. 
%Since each row of $A$ contains exactly $d$ ones, the maximum singular value of $A$ is $d$, and hence for any vector $v \perp \mathbbm{1}$, we have $|Av|_2 \leq \sigma_2|vz
Thus 
$$|Az|_2^2 \leq \sigma_2^2 |z|_2^2 \leq \left(\frac{\sigma_2}{d}\right)^2\frac{sm}{m - s}$$
as desired.
\end{proof}
\begin{corollary}\label{singular}
Let $A \in \mathbb{R}^{n \times m}$ be a graph assignment scheme corresponding to some $d$-regular graph $G$ with spectral expansion $\lambda$. Then for any set of $\lfloor{pm}\rfloor$ stragglers, there exists some decoding coefficients $w$ such that 
$$\frac{1}{n}|\alpha - \mathbbm{1}|_2^2 \leq \frac{2d - \lambda}{2d}\frac{p}{(1 - p)}.$$
\end{corollary}
\begin{proof}
Let $\mathcal{A}(G)$ denote the adjacency matrix of $G$ such that $\lambda_1(\mathcal{A}(G)) = d$ and $\lambda_2(\mathcal{A}(G)) = d - \lambda$. Observe that because $A$ has exactly two ones per column, we have $A^TA = \mathcal{A}(G) + dI,$ such that $\sigma_2(A) = \sqrt{\lambda_2(A^TA)} = \sqrt{2d - \lambda}.$ Applying \Cref{prop:generic_adv} implies that
$$\frac{1}{n}|\alpha - \mathbbm{1}|_2^2 \leq \frac{1}{n}\frac{2d - \lambda}{d^2}\frac{p}{(1 - p)}m.$$
Using the fact that $d = 2m/n$ concludes the proof.
\end{proof}
\begin{corollary}\label{adversarial_good_expander}
Let $A \in \mathbb{R}^{n \times m}$ be a graph assignment scheme corresponding to some $d$-regular graph $G$ with spectral expansion $\lambda$.  
%Consider a assignment scheme corresponding to a $d$-regular graph $G$ on $n$ data blocks and $m$ machines, where each block contains $m/n = d/2$ data points such that $N = m$.
Let $\lambda$ be the spectral gap of $G$ and suppose $\lambda = d - o(d)$. Then for any set of $\lfloor{pm}\rfloor$ stragglers, there exists some decoding coefficients $w$ such that $$\frac{1}{n}|\alpha - \mathbbm{1}|_2^2 \leq \frac{1 + o(1)}{2}\frac{p}{(1 - p)}.$$
\end{corollary}
\begin{proof}
Plugging in $\lambda = d - o(d)$ to \Cref{singular}, we have
$$\frac{1}{n}|\alpha - \mathbbm{1}|_2^2 \leq \frac{2d - \lambda}{2d}\frac{p}{(1 - p)} 
= \frac{ 1 + o(1) }{2} \frac{p}{1-p}$$
as desired.
\end{proof}
% \begin{remark}
% For any number of data points, by using a block size $N/m$ with a graph assignment scheme satisfying the conditions of \Cref{adversarial_good_expander}, we achieve $$\frac{1}{N}|\alpha - \mathbbm{1}|_2^2 \leq \frac{2d - \lambda}{2d}\frac{p}{(1 - p)} 
% = \frac{ 1 + o(1) }{2} \frac{p}{1-p}.$$
% \end{remark}
\begin{remark}[Tightness of \Cref{adversarial_good_expander}]
This bound is nearly tight for graph assignment schemes when $p$ is small. Indeed, for any graph assignment scheme on with $mp$ stragglers and replication factor $d$, we can adversarially choose the stragglers such that at least $\frac{mp}{d}$ data blocks are not held at any non-straggling machines.  Thus, for any decoding coefficients $\alpha$ we have
\[ \frac{1}{n} |\alpha - \mathbbm{1}|_2^2 \geq \frac{ mp }{dn} = \frac{p}{2} \]
using the fact that $nd = 2m$ for graph-based schemes.
\end{remark}

While our scheme improves by nearly a factor of two over the FRC of \cite{Tandon}, it is worse by an order of $d/8$ from the expander code of \cite{ExpanderCode}, which meets the lower bound in adversarial error up to constant factors for a replication factor of $d$ (See Table~\ref{comp_table}). We leave it as an open question to improve on our scheme for adversarial stragglers while maintaining our performance for random stragglers.

\section{Convergence with Random Stragglers}\label{convergence}

In this section, we bound the convergence rate of our coded gradient descent algorithm with random stragglers. This algorithm begins by distributing the data blocks according to the assignment matrix $A$. We additionally shuffle our assignment of data blocks to machines using a random permutation $\rho$. The iterative computation phase of the algorithm follows \Cref{star2} in the introduction, where the coefficients $w$ are given by the \em optimal \em decoding coefficients in \Cref{eq:optimal}. For the reader's convenience, we summarize the logical view of this algorithm with random straggers in the algorithm below.

\begin{algorithm}[tb]
  \caption{Gradient Coding with Optimal Decoding (GCOD): Logical View with Random Stragglers}
  \label{dgc}
\begin{algorithmic}
\Procedure{GCOD}{ $A, p, \theta_0, \gamma, \{f_i\}, k$}
\State \Comment{Distribution Phase}
\State $\rho \sim \text{Uniform}(\mathcal{S}_n)$ \Comment{$\rho$ is a random permutation}
  \For{$i=1$ {\bfseries to} $n$}
    \State Send $f_{\rho(i)}$ to all machines $j$ such that $A_{i,j} \neq 0$
  \EndFor
\State \Comment{Computation Phase}
\For{$t=1$ {\bfseries to} $k$}
\State    Parameter server: Send $x_{t - 1}$ to each machine
\For{Machine $j \in [m]$}
\State         $B_j \sim \text{Bernoulli}(p)$
\If{$B_j = 1$}
\State Machine $j$: Send $g_j = \sum_i{A_{ij}\nabla f_{\rho(i)}(\theta_{t - 1})}$ to parameter server.
\EndIf
\EndFor
\State Parameter server: Computes $w^* \in \argmin_{w: w_j = 0 \text{ if } B_j = 0}(|Aw - \mathbbm{1}|_2)$
\State Parameter server: $\theta_t \leftarrow \theta_{t - 1} - \gamma \sum_{j: B_j = 1}{w^*_jg_j}$\;
\EndFor
\Return{$\theta_k$}
\EndProcedure
\end{algorithmic}
\end{algorithm}

Note that in this algorithm, the optimal decoding vector $w^*$ computed by the parameter server might not be unique, but the vector $\Astar := Aw^*$ is unique. Indeed, for a straggler rate of $p$, let $A(p)$ be the random matrix which is a copy of $A$ with each column replaced with zeros independently with probability $1 - p$. Then $\Astar$ is the unique projection of the all-ones vector onto the space spanned by $A(p)$, namely \begin{equation}\label{alpha_def}\Astar = A(p)(A(p)^TA(p))^{\dagger}A(p)^T\mathbbm{1}.\end{equation}
Given a matrix $A$, let $P_{\Astar}$ be the distribution of the random vector $\Astar$ defined in \Cref{alpha_def}. Similarly let $P_{\Abarstar}$ be the distribution of $\Abarstar$, which we recall is defined to be the normalization $\Astar\frac{|\mathbbm{1}|_2}{|\mathbb{E}[\Astar]|_2}$. Then for any unbiased decoding scheme, GCOD($A, p, x_0, \gamma, \{f_i\}, k$) is stocastically equivalent to the gradient descent algorithm, SGD-ALG($P_{\Abarstar}, x_0, \gamma\mathbb{E}[\alpha_1], \{f_i\}, k$), given in Algorithm~\ref{sgd_alpha}.

\begin{algorithm}[tb]
 \caption{SGD-ALG($P_{\beta}, \theta_0, \gamma, \{f_i\}, k$)}\label{sgd_alpha}
\begin{algorithmic}
\Procedure{SGD-ALG}{$P_\beta, \theta_0, \gamma, \{f_i\}, k$}
%\INPUT $P_{\beta}, \theta_0, \gamma, \{f_i\}, k$
\State $\rho \sim \text{Uniform}(\mathcal{S}_n)$ \Comment{$\rho$ is a random permutation}
\For{$t=1$ {\bfseries to} $k$}
\State $\beta \sim P_{\beta}$
\State $\theta_t \leftarrow \theta_{t - 1} - \gamma \sum_{i}{\beta_{i}\nabla f_{\rho(i)}(\theta_{t - 1})}$
\EndFor
\Return{$\theta_k$}
\EndProcedure
\end{algorithmic}
\end{algorithm}

We provide convergence analysis of SGD-ALG in the following proposition, for distributions $P_{\beta}$ with $\mathbb{E}_{\beta \sim P_{\beta}}[\beta] = \mathbbm{1}$.

\begin{restatable}{proposition}{rstconv}\label{conv}
Let $f = \sum_i^{n}f_i$ be a $\mu$-strongly convex function with an $L$-Lipshitz gradient, and suppose each $f_i$ is convex, and all gradients $\nabla f_i$ are $L^\prime$-Lipshitz. Let $x^*$ be the minimizer of $f$. Let $\sigma^2 = \sum_i{|\nabla f_i(x^*)|_2^2}$.

Suppose we run the gradient descent as in Algorithm~\ref{sgd_alpha},
SGD-ALG($P_{\beta}, x_0, \gamma, \{f_i\}, k$), starting from $x_0$ for $k$ iterations with some step size $\gamma  \leq \frac{1}{sL^\prime + L}$ and some distribution $P_{\beta}$ such that $\mathbb{E}[\beta] = \mathbbm{1}$. Let $r := \frac{1}{n}\mathbb{E}\left[|\beta - \mathbbm{1}|_2^2\right]$, and $s := |\mathbb{E}[(\beta - \mathbbm{1})(\beta - \mathbbm{1})^T]|.$ Then 
\begin{equation}\label{conv_bd}\mathbb{E}\left[|x_k - x^*|^2_2\right] \leq \left(1 - 2\gamma\mu\left(1 - \gamma(sL^\prime  + L)\right)\right)^k|x_0 - x^*|_2^2 + \frac{\gamma r\left(1 + \frac{1}{n - 1}\right) \sigma^2}{\mu\left(1 - \gamma(sL^\prime  + L)\right)},\end{equation} where the expectation is over $\rho$ and $\{\beta^{(j)}: j < k\}$.
\end{restatable}

\begin{restatable}{corollary}{rstconvcor}\label{convcor}
For any desired accuracy $\epsilon$, we can choose a step size 
$$\gamma = \frac{\mu\epsilon}{2\mu\epsilon(sL^\prime + L)+ 2r\left(1 + \frac{1}{n - 1}\right)\sigma^2}$$ 
such that after $$k = 2\log(2\epsilon_0 /\epsilon)\left(\frac{sL^\prime}{\mu} + \frac{L}{
\mu} + \frac{r\left(1 + \frac{1}{n - 1}\right)\sigma^2}{\mu^2\epsilon}\right)$$ steps,
$\mathbb{E}\left[|x_k - x^*|^2_2\right] \leq \epsilon,$ where $\epsilon_0 = |x_0 - x^*|^2$.
\end{restatable}

Given our results on the decoding error $\mathbb{E}|\Astar - \mathbbm{1}|_2^2$ and covariance of $\Astar$ in \Cref{expmainfull}, we can use Proposition~\ref{conv} to bound the convergence time of coded gradient descent with a graph-based assignment scheme in the setting of random stragglers. This yields the following propostion.
\begin{proposition}\label{conv_full}
Let $f = \sum_i^{n}f_i$ be a $\mu$-strongly convex function with an $L$-Lipshitz gradient, and suppose each $f_i$ is convex, and all gradients $\nabla f_i$ are $L^\prime$-Lipshitz. Let $\theta^*$ be the minimizer of $f$, and define $\sigma^2 := \sum_i{|\nabla f_i(\theta^*)|_2^2}$.

Suppose we perform gradient coding with optimal decoding as in Algorithm~\ref{dgc} with an assignment matrix corresponding to a $d$-regular vertex-transitive graph with spectral gap $d - o(d)$, such that the number of machines $m = \frac{nd}{2}$. Let $p$ be the probability of a machine straggling.

Then for any desired accuracy $\epsilon$, we can choose some step size $\gamma$ such that after $$k = 2\log(\epsilon_0 /\epsilon)\left(\frac{L}{\mu} + \log^2(n)p^{2d - o(d)}\frac{L^\prime}{\mu} + \frac{p^{d - o(d)}\sigma^2}{\mu^2\epsilon}\right)$$ steps of gradient descent, we have
$\mathbb{E}\left[|\theta_k - \theta^*|^2_2\right] \leq \epsilon,$ where $\epsilon_0 = |\theta_0 - \theta^*|^2$.
\end{proposition}

\begin{remark}
Our results improve over black-box methods for establishing convergence of gradient descent (such as Theorem 34 in \cite{ExpanderCode}) for two reasons. First, we leverage the structure of the covariance matrix of $\Astar$ to control the dependence on the Lipshitz constants of gradients of each data block. Second, by shuffling the data blocks before assignment, we are able to bound $\mathbb{E}\left[\left|\sum_{i = 1}^n{\Astar_i \nabla f_i(\theta^*)}\right|^2_2\right]$ much more tightly than the naive bound $\mathbb{E}\left[|\Astar - \mathbb{E}[\Astar]|_2^2\right]\sum_i{|\nabla f_i(\theta^*)|_2^2}.$ This quantity controls the constant that appears in front of $1/\epsilon$ in \Cref{conv_full}. These improvements allow us to converge up to a factor of $n$ faster than black-box methods, though the exact improvement depends on the functions $f_i$.
\end{remark}

\begin{remark}
The step size used in Propostion~\ref{conv_full} scales inversely with the quantity $E|\Abarstar - \mathbbm{1}|_2^2$, which controls the variance of the gradient estimate. Choosing a step size inversely proportional to this variance term is common in other work (eg. \cite{ExpanderCode}, \cite{SGD}).
\end{remark}
\begin{remark}
 \Cref{conv_full}  relies on the assignment scheme being unbiased.  However, it can be applied more generally at the expense of doubling the computation load: we show in \Cref{debias} in Appendix~\ref{apx:debias} how to debias any assignment scheme.
\end{remark}

We provide proofs of \Cref{conv} and \Cref{convcor} in Appendix~\ref{apx:conv}. Combining \Cref{conv} with \Cref{expmainfull} yields \Cref{conv_full}.

\section{Convergence with Adversarial Stragglers}\label{sec:adv_convergence}
In this section, we show that with adversarial stragglers, coded gradient descent can converge down to a noise floor which scales with the maximum value $|\Astar - \mathbbm{1}|_2^2$.
\begin{restatable}{proposition}{rstadv}\label{prop:formal_adv_conv}
Let $f = \sum_i^{n}f_i$ be a $\mu$-strongly convex function with an $L$-Lipshitz gradient, and suppose each $f_i$ is convex, and all gradients $\nabla f_i$ are $L^\prime$-Lipshitz. Let $x^*$ be the minimizer of $f$, and define $\sigma^2 := \sum_i{|\nabla f_i(x^*)|_2^2}.$ Suppose we perform gradient descent with the update \begin{equation}\label{eq:adv_grad_step}
x_{k + 1} = x_k - \gamma \sum_i{\alpha_i^{(k)} \nabla f_i(x_k)},  
 \end{equation} such that at each iteration, $\left|\alpha^{(k)} - \mathbbm{1}\right|_2^2 \leq r^2$.  Let $a = 1 - \frac{r\sqrt{L^\prime}}{\sqrt{\mu}}$, and suppose $a > 0$. For any $0 < \epsilon \leq 1$, we can choose a constant step size of 
$$\gamma = \frac{\epsilon a\mu}{4(L^2 + 2rLL^\prime + 4r^2(L^\prime)^2)}$$ such that for some 
\begin{equation}\label{k_val}
k \leq \frac{2(1 + \epsilon)\log\left(\frac{2a^2\mu^2|y_0|_2^2}{(1 + \epsilon)^2r^2\sigma^2}\right)}{3\gamma a \mu \epsilon} = \frac{4(1 + \epsilon)(L^2 + 2rLL^\prime + 4r^2(L^\prime)^2)\log\left(\frac{2a^2\mu^2|x_0 - x_*|_2^2}{(1 + \epsilon)^2r^2\sigma^2}\right)}{3a^2\mu^2 \epsilon^2}
\end{equation}
iterations, we have 
\begin{equation}\label{threshold}
|x_k - x_*|_2 \leq (1 + \epsilon)\frac{r\sigma}{a\mu}.
\end{equation}
\end{restatable}
Plugging in $\epsilon = 1$, we obtain the following corollary:
 \begin{corollary}\label{prop:informal_adv_conv}
 Let $f = \sum_i^{n}f_i$ be a $\mu$-strongly convex function with an $L$-Lipshitz gradient, and suppose each $f_i$ is convex, and all gradients $\nabla f_i$ are $L^\prime$-Lipshitz. Let $\theta^*$ be the minimizer of $f$, and define $\sigma^2 := \sum_i{|\nabla f_i(\theta^*)|_2^2}$.
 
 Suppose we perform gradient descent with the update $\theta_{k + 1} = \theta_k - \gamma \sum_i{\alpha_i^{(k)} \nabla f_i},$ such that at each iteration, $\left|\alpha^{(k)} - \mathbbm{1}\right|_2^2 \leq r$. Assume $\mu > rL^\prime$.

Then we can choose some step size $\gamma$ such that for some $$k \leq \frac{3(L + 2\sqrt{r}L^\prime)^2\log\left(\frac{\mu^2\epsilon_0}{2r\sigma^2}\right)}{\left(\mu - \sqrt{\mu rL^\prime}\right)^2},$$ we have
$|\theta_k - \theta^*|^2_2 \leq \frac{4r\sigma^2}{\left(\mu - \sqrt{\mu rL^\prime}\right)^2},$ where $\epsilon_0 = |\theta_0 - \theta^*|^2$.
 \end{corollary}
Plugging in our decoding error bound on $|\Astar - \mathbbm{1}|_2^2$ from \Cref{adversarial_good_expander} shows that we can converge to a noise floor of $$\frac{4\sigma^2n(2d - \lambda)p}{\mu(\sqrt{2d(1-p)\mu} - \sqrt{n(2d - \lambda)pL^\prime})^2}.$$
\begin{remark}
In the case of a linear regression problem, where $f_i(\theta) = (a_i^T\theta - b_i)^2$, if the vectors $a_i \sim N(0, I_k)$ and values $b_i = \langle{a_i, \theta}\rangle + z_i$ for $z_i\sim N(0, \zeta^2)$, we expect to have $\mu \approx 2N\left(1 - \sqrt{\frac{k}{N}}\right)^2$, and $L' < 10\max\left(\frac{N}{n}, k\right)$ with high probability \cite{rmt_notes}. Assuming $N > 4k$, \Cref{prop:informal_adv_conv} shows that for $p \leq 0.05\min(1,\frac{4N}{nk})$, we can converge to a noise floor of $\frac{20pn^2k^2\zeta^2}{N^2}$.
\end{remark}

We defer the proof of this proposition~\label{prop:formal_adv_conv} to Appendix~\ref{apx:adv_conv}.
 
\section{Experiments}\label{sec_simulations}
In this section, we demonstrate empirically that our scheme achieves near-optimal error $\mathbb{E}\left[\left|\alpha^* - \mathbbm{1}\right|_2^2\right]$ in the random stragglers model, and further that our scheme converges in fewer iterations with optimal decoding coefficients than with fixed coefficients. We demonstrate the advantage of our coded approach to uncoded gradient descent, and compare the convergence of our scheme to previous work on coded gradient descent. We run our convergence experiments in both a simulated setting and in a distributed compute cluster, which we describe in more detail in Section~\ref{exp:conv}.

Our experiments are conducted in two different parameter regimes. The first regime is in a setting with $m = 24$ machines and $N = 60000$ data points, where each data point is replicated $d = 3$ times such that the compuational load is $\ell = 7500$. The second regime has $m = 6552$ machines and $N = 6552$ data points, and each data point is replicated $d = 6$ times such that $\ell = 6$. In the first regime, for our coded approach, we use an assigment matrix $A_1$ which corresponds to a random $3$-regular graph on $n=16$ vertices with $m=24$ edges. While this graph is not vertex-transitive, and hence may result in a biased approximation of the gradient, it represents a practical regime with $24$ machines, and is with high probability a good expander. In the second regime, we use an assignment matrix $A_2$ which corresponds to the degree $6$ LPS expander of \cite{LPS} on $n=2184$ vertices with $6552$ edges. We chose this graph because it is the smallest vertex-transitive expander.

In each of these two parameter regimes, we compare the performance the performance of the following four coding schemes:
\begin{enumerate}
\item Our coded approach using $A_1$ or $A_2$ with optimal decoding.
\item Our coded approach using $A_1$ or $A_2$ with fixed decoding.
\item The coded approach of \cite{ExpanderCode}. With $m = 24$, we use the for the assignment matrix the adjacency matrix of a random graph on $24$ vertices of degree $3$. In this regime, we conduct optimal decoding. With $m = 6552$, we use a random graph on $6552$ vertices of degree $6$. Due to the computational complexity of decoding, we use fixed decoding coefficients in this regime.
\item The coded approach of \cite{ErasureHead} (which uses an FRC) with replication factor $d = 3$ or $d = 6$. 
\end{enumerate}

For fixed decoding, we use the decoding vector $w^{\text{fixed}}$ where $w^{\text{fixed}}_j = 0$ if $j$ is a straggler and otherwise $w^{\text{fixed}}_j = \frac{1}{d(1 - p)}$. In this manner, we have $\mathbb{E}[Aw^{\text{fixed}}] = \mathbbm{1}$. 

\subsection{The decoding error $\mathbb{E}[|\Abarstar - \mathbbm{1}|_2^2]$}
In our first set of experiments, shown in \Cref{Covar_FRC_LPS}, we compare the decoding error $\mathbb{E}\left[\frac{1}{N}|\Abar - \mathbbm{1}|_2^2\right]$ and the norm of the covariance $|\mathbb{E} (\Abar - \mathbbm{1})(\Abar - \mathbbm{1})^T|_2$ under a $p$ fraction of random stragglers for the first four schemes above. In \Cref{Covar_FRC_LPS}(a)(b), we consider the first regime with $d = 3$, and in \Cref{Covar_FRC_LPS}(c)(d) the second regime where $d = 6$. We note that the FRC of \cite{Tandon} achieves the theoretical optimum of $\frac{1}{N}\mathbb{E}|\Abarstar - \mathbbm{1}|_2^2 = \frac{p^d}{1 - p^d}$, and hence we plot this optimum in place of the results from the FRC. Similarly, for the FRC assigment,we have $$\left|\mathbb{E}\left[(\Abarstar - \mathbbm{1})(\Abarstar - \mathbbm{1})^T\right]\right|_2 = \frac{\ell}{N}\mathbb{E}|\Abarstar - \mathbbm{1}|_2^2.$$ 
This equation holds because the covariance matrix has zeros everywhere except in entries corresponding to two data points in the same block. In all these three settings, we use $N = n$ data points, such that the computational load is $\ell = 6$.

\begin{figure}
\setlength{\tabcolsep}{0pt}
\begin{tabular}{cccc}
  \includegraphics[width=4.3cm]{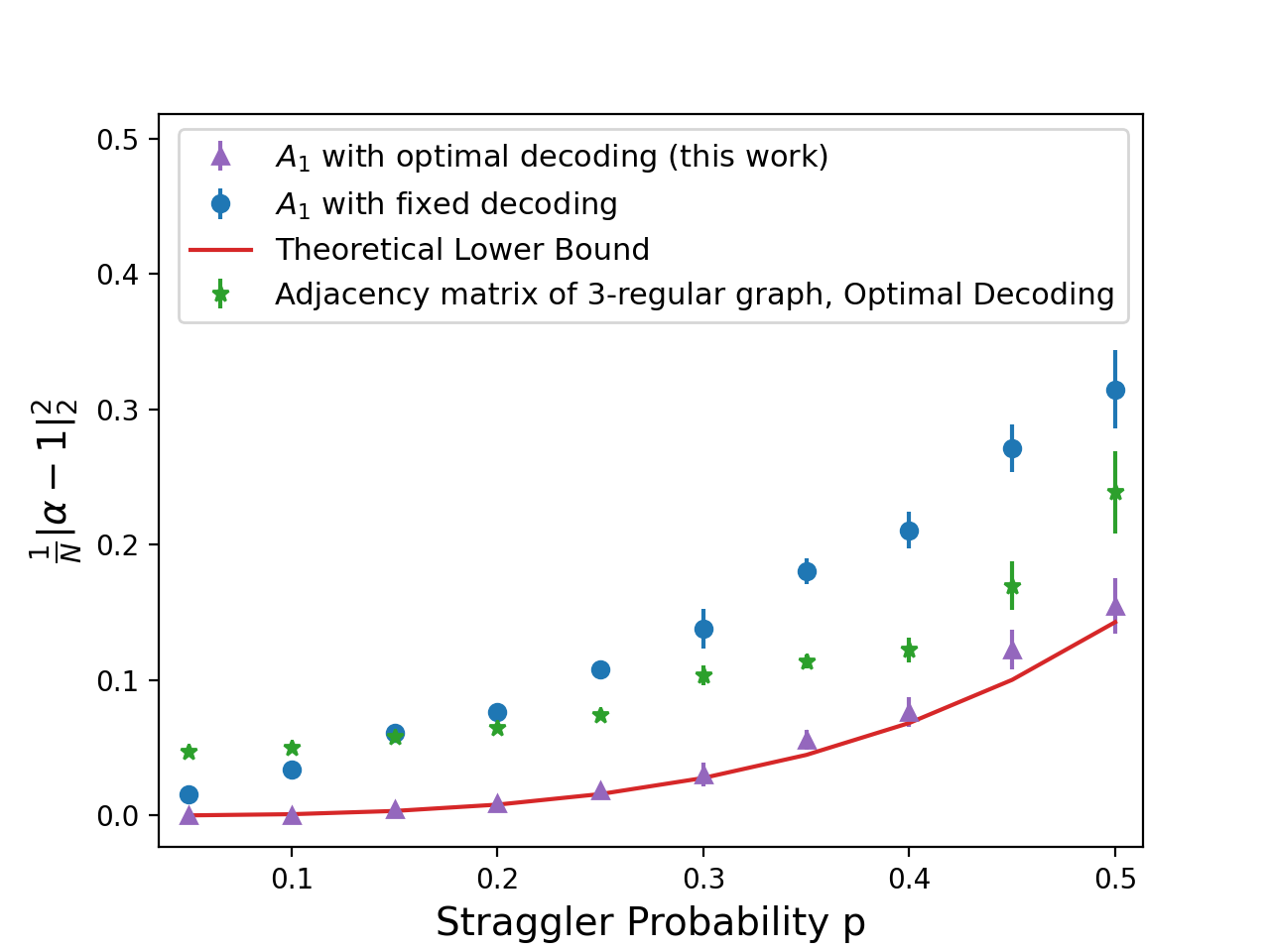} &  \includegraphics[width=4.3cm]{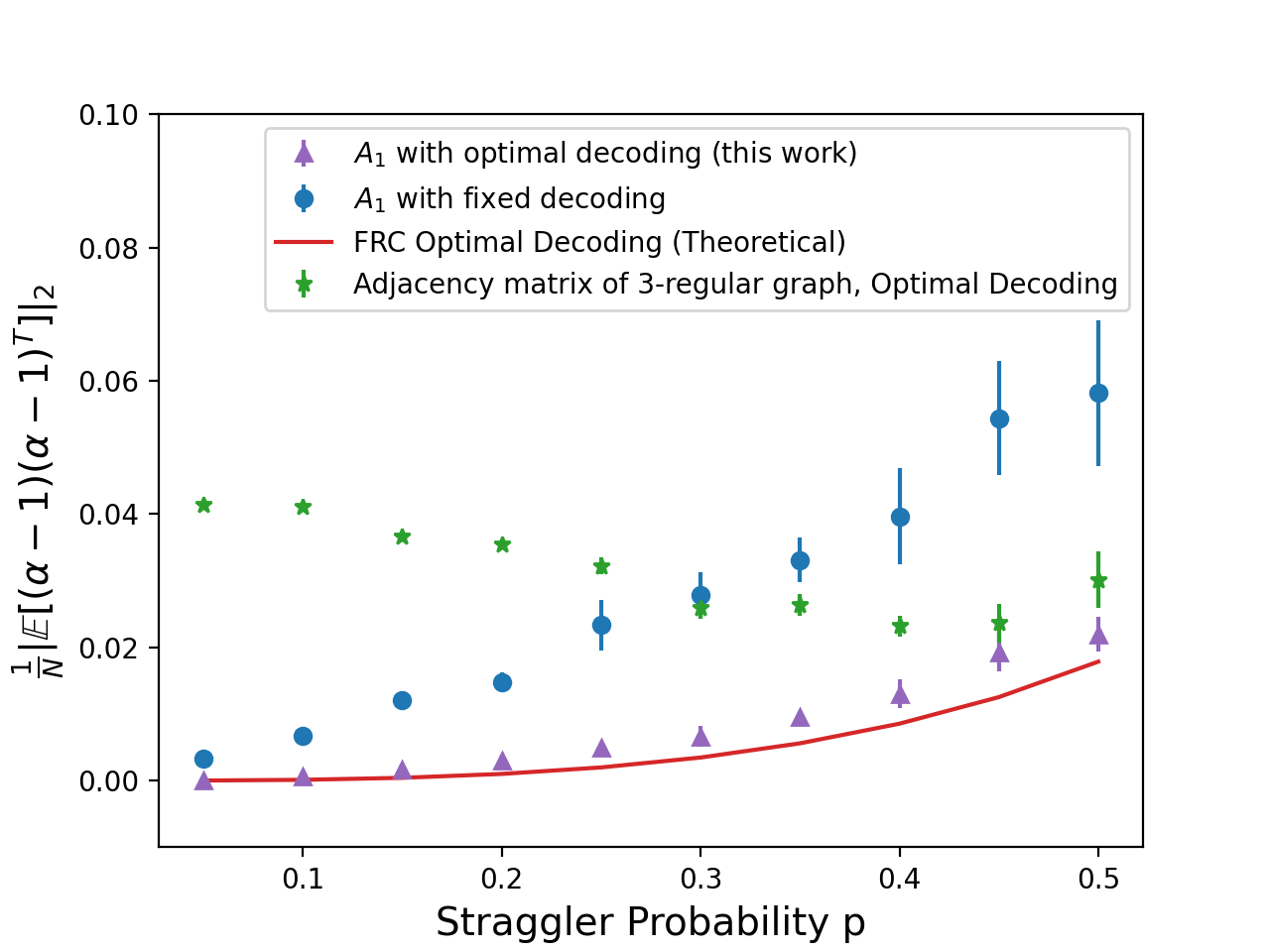} & \includegraphics[width=4.3cm]{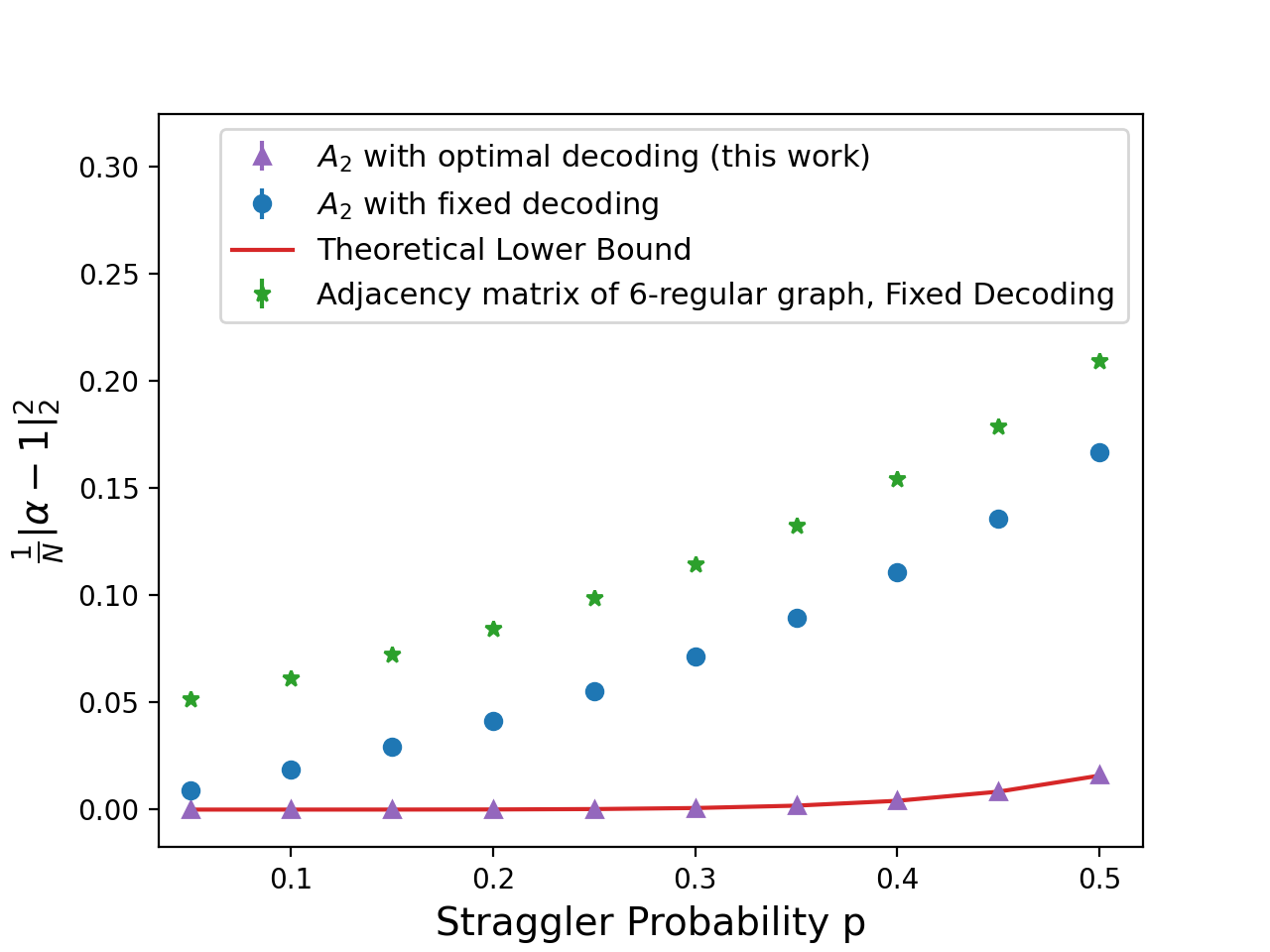} &   \includegraphics[width=4.3cm]{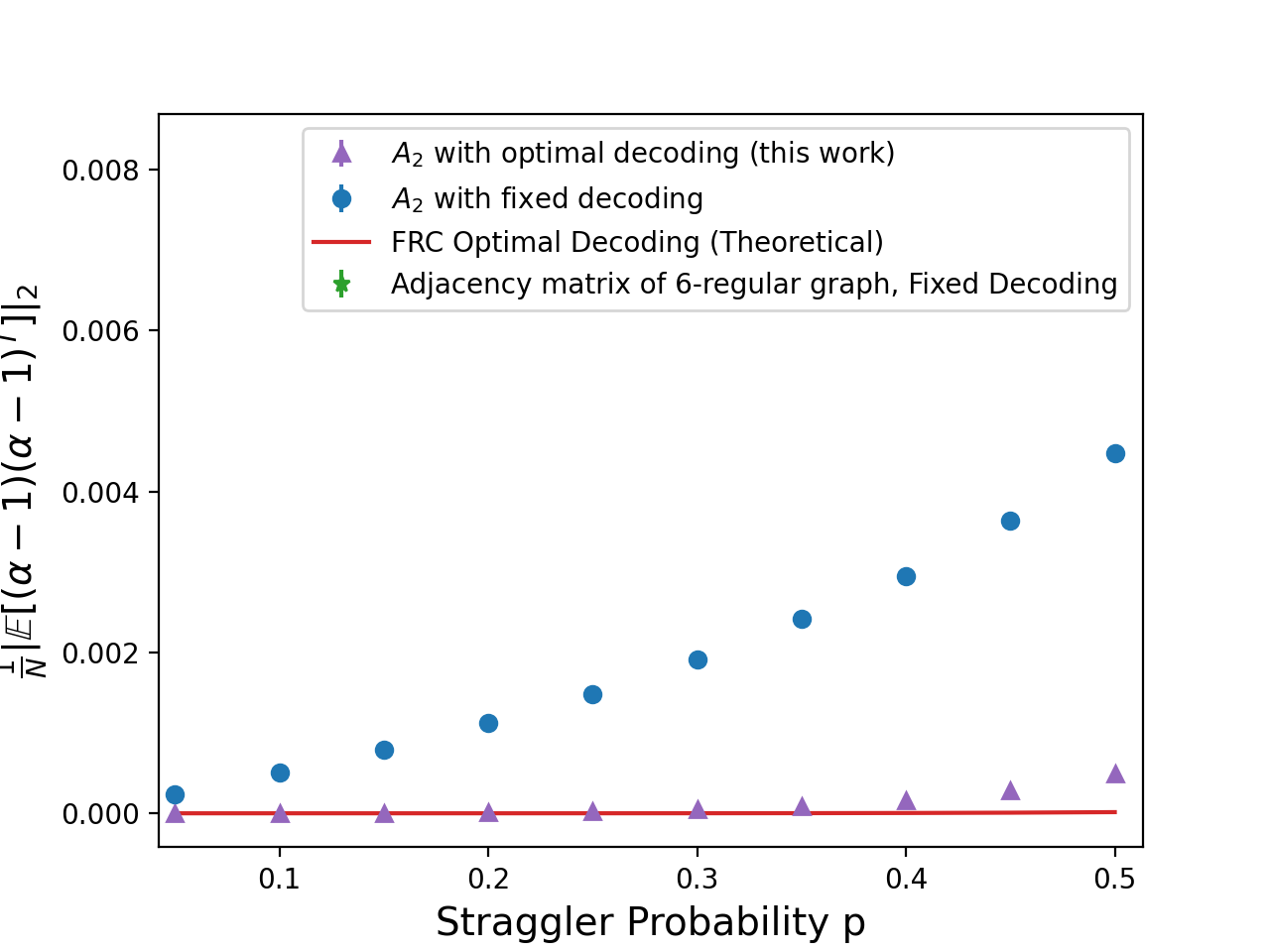}\\
(a)  & (b) & (c) & (d) \\[6pt]
\end{tabular}

  \caption{Comparison of variance and covariance of $\alpha$. All values estimated over 50 runs. Error bars are for the standard deviation of the empirical values, evaluated over 5 experiments.
 (a) The empirical expectation $\frac{1}{N}\mathbb{E}\left[\left|\Abar - \mathbbm{1}\right|_2^2\right]$ with $m = 24$, $d = 3$.
 (b) Norm of the empirical covariance matrix $\frac{1}{N}\left|\mathbb{E}\left[(\Abar - \mathbbm{1})(\Abar - \mathbbm{1})^T\right]\right|_2$ with $m = 24$, $d = 3$.
 (c) The empirical expectation $\frac{1}{N}\mathbb{E}\left[\left|\Abar - \mathbbm{1}\right|_2^2\right]$ with assignment $A_2$  with $m = 6552$, $d = 6$.
 (d) Norm of the empirical covariance matrix $\frac{1}{N}\left|\mathbb{E}\left[(\Abar - \mathbbm{1})(\Abar - \mathbbm{1})^T\right]\right|_2$ with $m = 6552$, $d = 6$. The points for the Expander code do not appear as they are significantly higher.
 }
  \label{Covar_FRC_LPS}
\end{figure}

\Cref{Covar_FRC_LPS} demonstrates that for both assignments $A_1$ and $A_2$, our scheme with optimal coefficients achieves near-optimal error $\mathbb{E}\left[\left|\Abarstar - \mathbbm{1}\right|_2^2\right]$ for small values of $p$, and significantly outperforms fixed coefficient decoding and the approach of \cite{ExpanderCode}.

\subsection{Convergence of Coded Gradient Descent}\label{exp:conv}
We compare the performance of coded gradient descent in the four coding schemes listed above in addition to an uncoded scheme which ignores stragglers. 

\textbf{Data.} We run gradient descent on a least squares problem $\min_{\theta} |X\theta - Y|_2^2,$ where $X \in \mathbb{R}^{N \times k}$ is chosen randomly with i.i.d. rows from $\mathcal{N}(0, \frac{1}{k}I_k)$, and $\theta \sim \mathcal{N}(0, I_k)$. The observations $Y$ are noisy observations of the form $Y = X\theta + Z$, where $Z \sim \sigma^2\mathcal{N}(0, I_N)$. In our first parameter regime with $m = 24$, we use $N = 60000$, $k = 20000$, and $\sigma = 100$. In our second parameter regime with $m = 6552$, we use $N = 6552$, $k = 200$, and $\sigma = 1$. We initialize $\theta$ at the origin, and let $\theta_*$ be the minimizer $(X^TX)^{-1}X^TY$. 

\textbf{Platform and Implementation.} In the first regime, we run our experiments on $m=24$ processors in Stanford's high compute cluster Sherlock, which contains any of the following four processors: Intel E5-2640v4, Intel 5118, AMD 7502, or AMD 7742. We implement the algorithms in Python using MPI4py, an open-source MPI implementation. At each iteration, the PS waits to receive gradient updates from the first $\lceil{m(1 - p)}\rceil$ processors using MPI.Request.Waitany. Then the PS computes optimal or fixed decoding coefficents (as specified by the scheme), takes a gradient step, and sends the next iterate $\theta$ to all of the processors. We plot the error $|\theta_t - \theta_*|^2$ after $50$ iterations in Figure~\ref{sherlock}. We start timing once the data has been loaded and the first iteration starts.

In the second regime with $m = 6552$ machines (which is too large for us to test on the Sherlock cluster) we simulate coded gradient descent on a single machine by computing the gradients update used that would be used in the presence of a specified set of stragglers.  We artificially select the stragglers independently with probability $p$. Precisely, our simulations implement Algorithm~\ref{sgd_alpha} with a specific input distribution $P_{\beta}$ which depends on the coding scheme. Recall that Algorithm~\ref{sgd_alpha} is stochastically equivalent to Algorithm~\ref{dgc} if the input distribution $P_{\beta}$ equals the distribution of $\Astar$. Hence for optimal decoding with an assignment matrix $A$, we let the input $P_{\beta}$ to Algorithm~\ref{sgd_alpha} be the distribution of $\Astar = A(p)(A(p)^TA(p))^{\dagger}A(p)^T\mathbbm{1}$ given by \Cref{alpha_def}. Recall here that $A(p)$ is the matrix $A$ where each column is deleted with probablility $p$. That is, at each iteration, we randomly sample $\beta$ to be this random vector. For fixed decoding with assignment matrix $A$, we let $P_{\beta}$ be the distribution of $Aw^{\text{fixed}}$. We plot the error $|\theta_t - \theta_*|^2$ after $50$ iterations in Figure~\ref{GD_Decoding}. As per Remark~\ref{uncoded} below, in the uncoded approach, we do $300$ iterations.

\begin{remark}\label{uncoded}
If the same number of machines are used in both coded and uncoded approaches, but the coded approach has a replication factor of $d$, then each machine in the coded approach has a gradient computation that is $d$ times bigger. We compensate for this by performing $d$ times as many iterations in the uncoded scheme. Note that if the communication time is the bottleneck, then each iteration of coded gradient descent will take \em less \em than $d$ times as long as an iteration of uncoded gradient descent: indeed, the communication times should be the same, while the computation time should increase by a factor of $d$. In this case, we expect the advantage of our approach over an uncoded approach to be greater than our simulations suggest.
\end{remark}

To be fair to all algorithms, for all experiments discussed, we use a grid search to find the best step size. We give more details and show the step size chosen by this grid search in Table~\ref{step_size_chart} in Appendix~\ref{apx:hp} for all algorithms discussed below. 

We observe that our algorithm substantially outperforms the expander code and the uncoded approach, and converges to error comparable with the FRC of \cite{Tandon} (we recall that the FRC of \cite{Tandon} achieves the optimal $\mathbb{E}\left[|\Abarstar - \mathbbm{1}|_2^2\right]$ for random stragglers, but is substantially sub-optimal for worst-case stragglers).  We note that our algorithm in many cases even outperforms the FRC: indeed, Figure~\ref{sherlock}(a) demonstrates faster convergence, and the table in Figure~\ref{sherlock}(b) shows that the final error is typically much smaller for our algorithm than for the FRC on the Sherlock cluster.  We conjecture that our algorithm is able to outperform the FRC (the theoretical optimum) on a real cluster because the assumption of i.i.d. stragglers is not perfectly correct: indeed, we observe that which machines are straggling tends to stay stagnant throughout a run.  We conjecture that the comparatively better performance of our algorithm on worst-case stragglers (relative to the FRC) gives it an advantage in such settings.

% These plots show that optimal decoding has MSE at least $\frac{1}{3p^{2}}$ times smaller than fixed coefficient decoding for $A_1$ (Figure~\ref{GD_Decoding}(b)). and $\frac{1}{2p^4}$ for $A_2$ (Figure~\ref{GD_Decoding}(b)). Further, for $d = 3$, our approach after $50$ iterations has MSE roughly $1/(10p^2)$ times smaller than the uncoded strategy after $60$ seconds. For $d = 6$, our approach after $50$ iterations has MSE roughly $1/(13p^4)$ times smaller than the uncoded strategy after $300$ iterations. 

% We observe, as expected, that with random stragglers, our approach performs similarly to the FRC of \cite{Tandon} (we recall that the FRC of \cite{Tandon} achieves the optimal $\mathbb{E}\left[|\Abarstar - \mathbbm{1}|_2^2\right]$ for random stragglers, but substantially sub-optimal results for worst-case stragglers). Our approach has MSE $4$ times smaller than that of \cite{ExpanderCode} for $d = 3$, and $1/(4p^2)$ times smaller for $d = 6$.

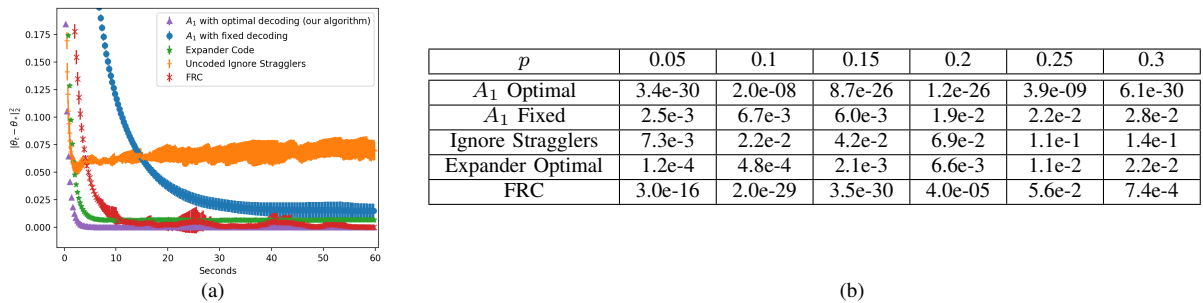
\begin{figure}\footnotesize

\begin{tikzpicture}
\node at (0, 0) {\includegraphics[width=5.5cm]{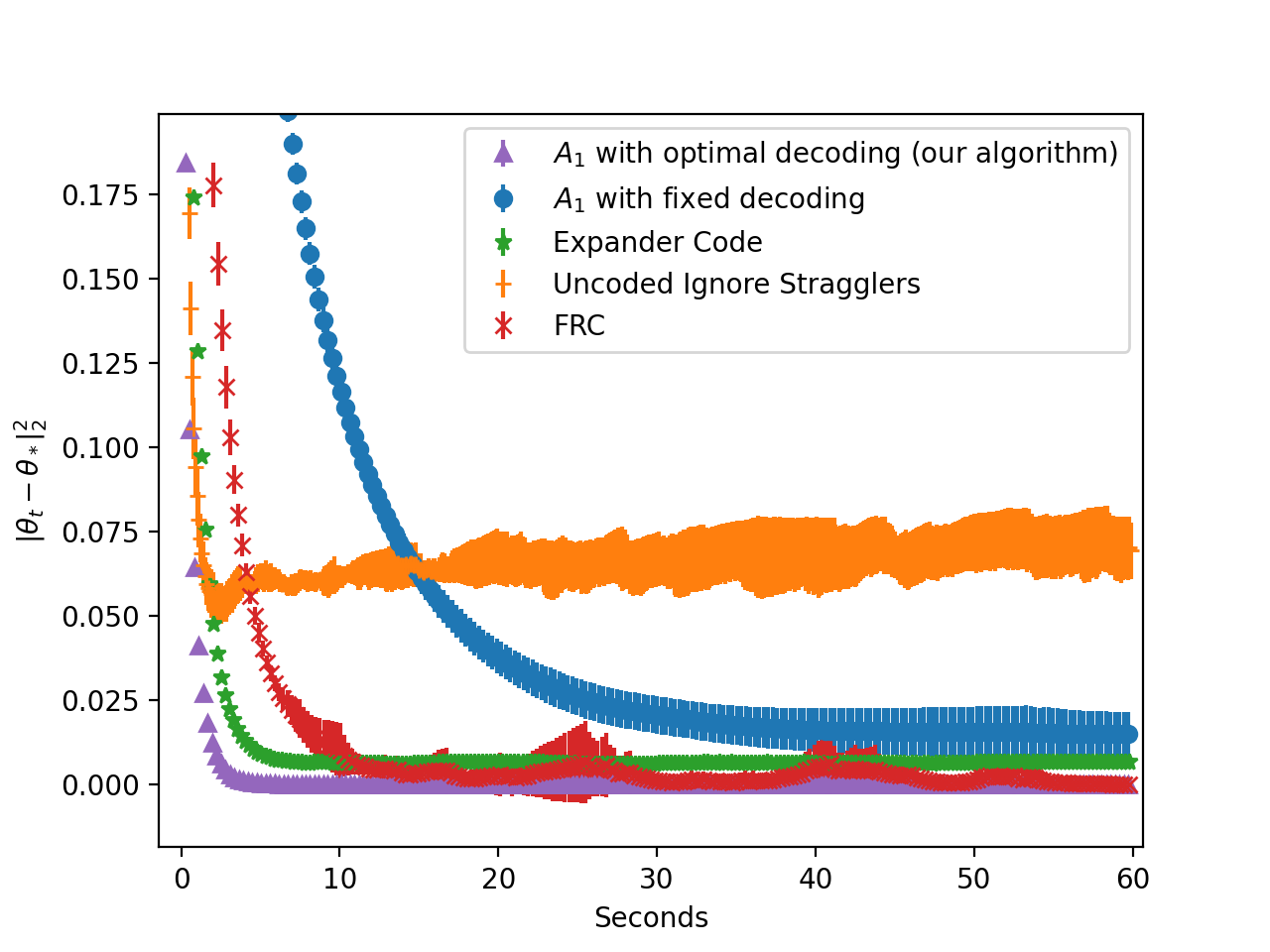}};
\node at (8, 0) { \begin{tabular}{|c|c|c|c|c|c|c|}
 \hline
 $p$ & 0.05 & 0.1 & 0.15 & 0.2 & 0.25 & 0.3 \\
 \hline
 \hline
$A_1$ Optimal & 3.4e-30 & 2.0e-08 & 8.7e-26 & 1.2e-26 & 3.9e-09 & 6.1e-30\\
\hline
$A_1$ Fixed & 2.5e-3 & 6.7e-3 & 6.0e-3 & 1.9e-2 & 2.2e-2 & 2.8e-2 \\
\hline
Ignore Stragglers & 7.3e-3 & 2.2e-2 & 4.2e-2 & 6.9e-2 & 1.1e-1 & 1.4e-1 \\
\hline
Expander Optimal & 1.2e-4 & 4.8e-4 & 2.1e-3 & 6.6e-3 & 1.1e-2 & 2.2e-2 \\
\hline
FRC & 3.0e-16 & 2.0e-29 & 3.5e-30 & 4.0e-05 & 5.6e-2 & 7.4e-4 \\
\hline
\end{tabular}};
\node at (0, -2.2) {(a)};
\node at (8.5, -2.2) {(b)};
% (a) & (b) \\[6pt]
\end{tikzpicture}

\caption{Comparison of coded gradient descent on a distributed cluster with $m = 24$, $N = 60000$. (a) Convergence of gradient descent with $p = 0.2$. (b) $|\theta_t - \theta_*|_2^2$ after $60$ seconds. Values are averaged over 8 runs, with error bars for standard deviation.
  }\label{sherlock}
\end{figure}

\begin{figure}
\begin{tabular}{cc}
 \includegraphics[width=8cm]{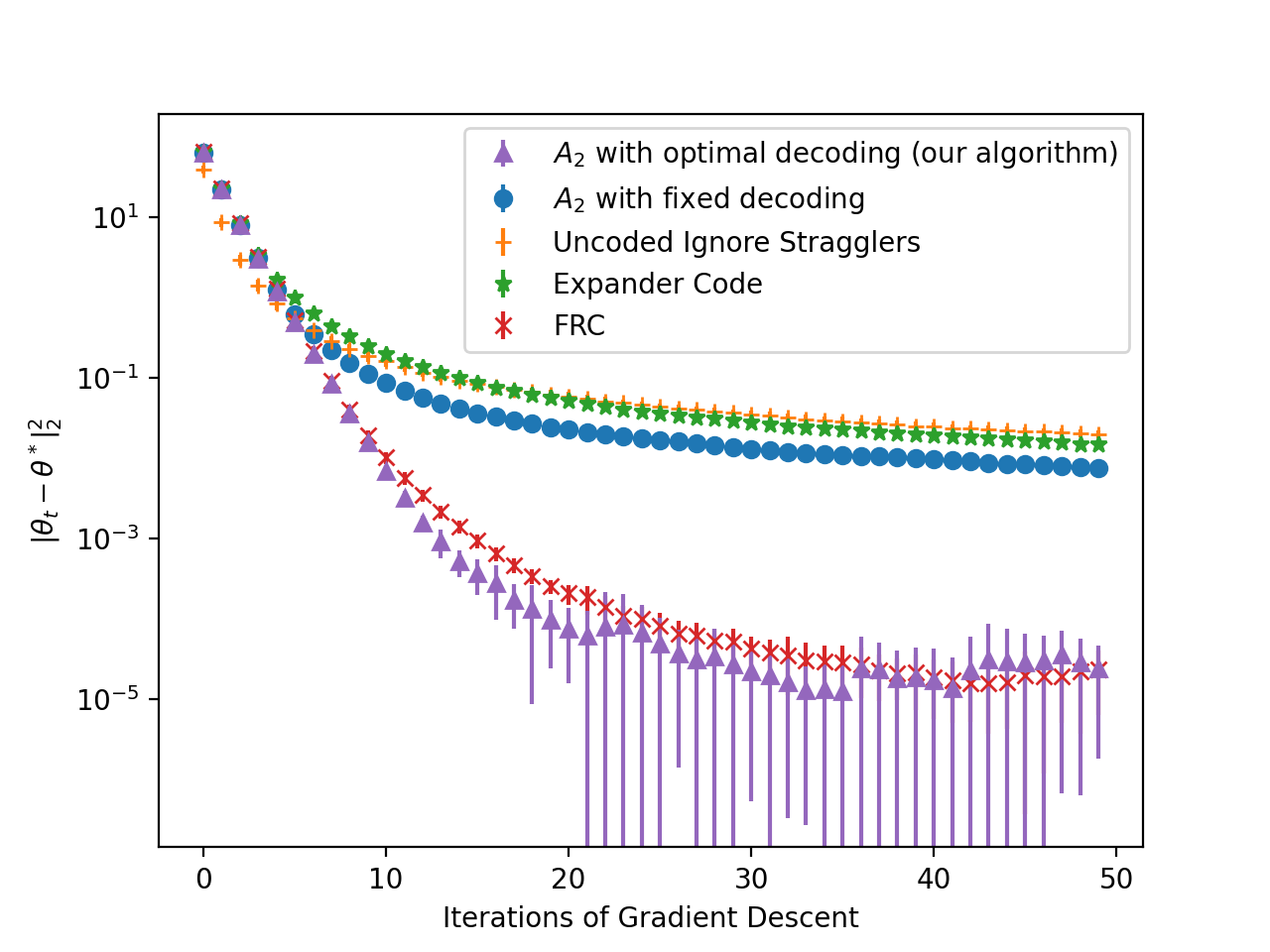} &   \includegraphics[width=8cm]{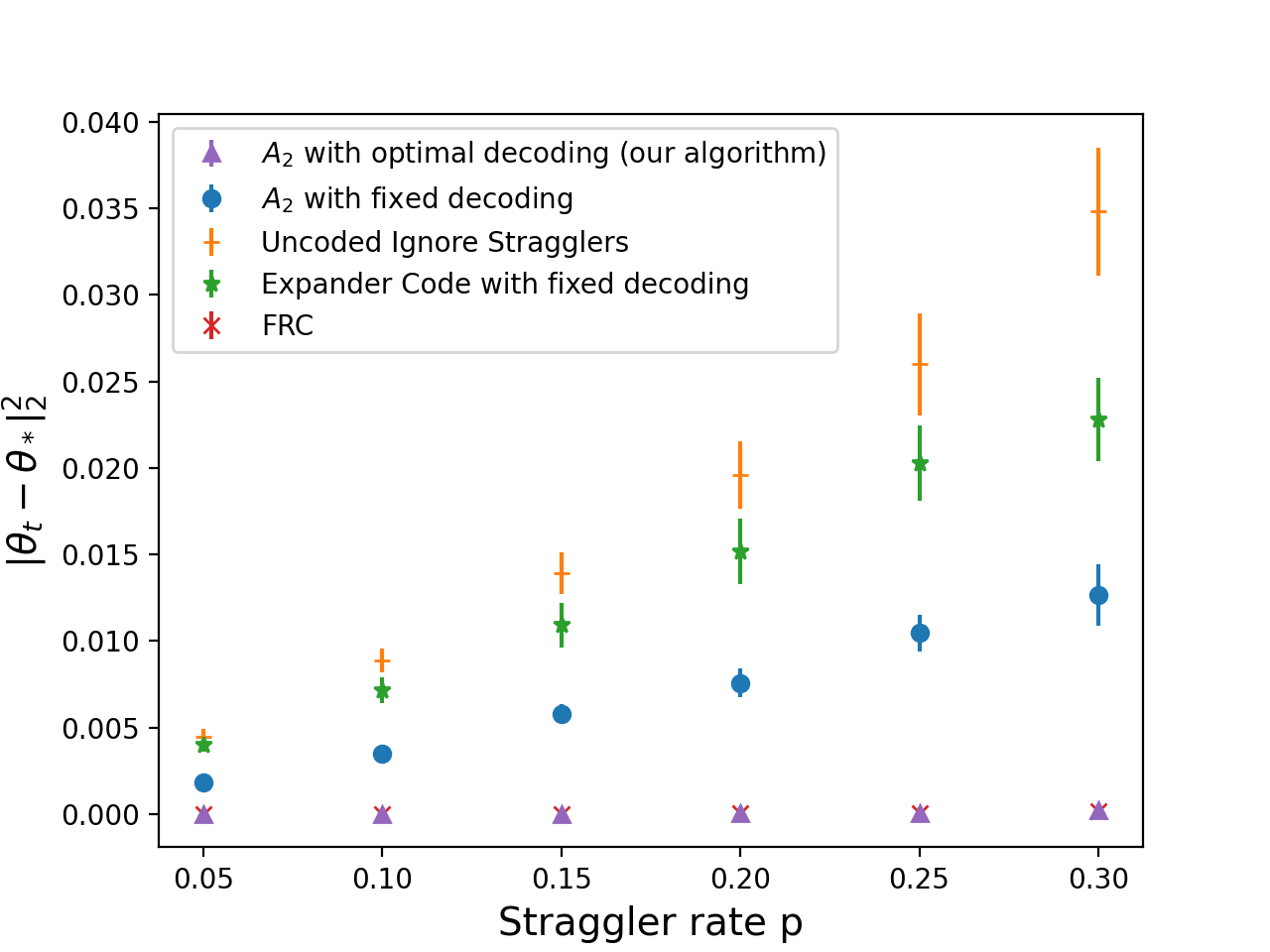} \\
(a) & (b)  \\[6pt]
\end{tabular}

  \caption{Comparison of simulated gradient descent algorithms with $m = 6552$, $N = 6552$. (a) Convergence of gradient descent with $p = 0.2$; the uncoded approach uses 6x as many iterations as shown. Values are averaged over 20 runs, with error bars for standard deviation. (b) $|\theta_t - \theta_*|_2^2$ after $50$ iterations. Values are averaged over 20 runs, with error bars for standard deviation.
  }
  \label{GD_Decoding}
\end{figure}

% \subsubsection{Comparison to Uncoded Gradient Descent and State-of-the-art}

% \begin{figure}

% \includegraphics[width=8cm]{Uncoded_GD.png}
% \includegraphics[width=8cm]{Uncoded_all_p.png}

% \caption{MSE to minimizer in coded gradient descent versus an uncoded approach that ignores stragglers.  Both settings have $m = 90$ machines and $N = 7200$ data points. All values are averaged over 100 runs, with error bars for standard deviation. (a) Convergence of gradient descent for straggler rate $p = 0.2$ over $50$ iterations for the coded assignment $A_1$ and $150$ iterations for the uncoded assignment. (b) $\theta_t - \theta_*|_2^2$ after 50 iterations with $A_1$ and 150 iterations uncoded.}
%  \label{uncoded}
% \end{figure}

% \begin{figure}
% \includegraphics[width=15cm]{SOTA.png}
% \caption{Comparison to state-of-the-art (Expander code of \cite{ExpanderCode} and FRC of \cite{Tandon}): Convergence of coded gradient descent for straggler rate $p = 0.2$ with $m = 90$ machines, $N = 7200$ data points, and replication factor $d=3$. All values averaged over 100 runs, with error bars for standard deviation.}
% \label{fig:SOTA}
%  % \label{GD_Decoding}
% \end{figure}

\section{Conclusion}\label{conclusion}
In this work, we present an approximate gradient coding scheme based on expander graphs, which performs well both in the adversarial and random straggler settings. We show how to analyze the optimal decoding error of our codes by relating $\Astar$ to the connected components in a randomly sparsified graph. 
We give provable convergence results in both the adversarial and random straggler settings.
%In the random straggler setting, we are able to bound the covariance matrix of $\Astar$ to yield good convergence bounds for gradient descent. In the adversarial setting, we give the first result on convergence of coded gradient descent.
We conclude with a few open questions.
\begin{enumerate}
    \item We have developed a technique for controlling the 
variance and covariance of the random variable $\Astar$ generated by the optimal decoding coefficients when each machine holds two data blocks, by analyzing the sparsification of random graphs.  It is an interesting open question to extend our techniques, or develop new ones, to work for a larger number of data blocks per machine.
\item 
While our scheme gives the best known error $|\Astar - \mathbbm{1}|_2^2$ in the adversarial setting given near-optimal error against random stragglers, it could be improved.  
Is there a coding scheme which achieves near-optimal error $|\Astar - \mathbbm{1}|_2^2$---that is, decaying like $p^{d - o(d)}$---while simultaneously achieving near-optimal adversarial error---that is, decaying like $\frac{1}{d}$? 
\end{enumerate}

% \section*{Acknowledgements} 
% We thank Tavor Baharav for helpful comments and suggestions.

\bibliographystyle{plain}
\bibliography{refs}

\pagebreak
\appendix
\begin{appendices}
\setcounter{page}{1}
\section{Lower bound on distance to $\mathbbm{1}$ for random stragglers}
For a matrix $A$, We use the notaton $nnz(A)$ to denote the number of non-zero entries in $A$.

\begin{proposition}[Fixed Decoding Lower Bound]\label{naive}
Consider any assignment scheme $A$ with $m$ machines $n$ data blocks with $n \leq nnz(A) \leq dn$. Suppose we use a fixed decoding coefficient scheme that yields an unbiased gradient, that is:
\begin{enumerate}
    \item For some $\hat{w}$, we use the decoding vector
$$w_j = \begin{cases}\hat{w}_j & \text{ Machine $j$ does not straggle} \\
0 & \text{ Machine $j$ straggles }\end{cases}.$$
\item $\mathbb{E}[\alpha] = c\mathbbm{1}$ for some $c$.
\end{enumerate}
Then
% $$(1 - p)w_j\sum_{j \in [m]}A_{ij} = \mathbb{E}[\alpha_i] = 1,$$ 
$$\frac{1}{n}\mathbb{E}\left[|\Abar - \mathbbm{1}|_2^2\right] \geq \frac{p}{d(1 - p)},$$ and 
$$|\mathbb{E}[\left(\Abar- \mathbbm{1}\right)\left(\Abar - \mathbbm{1}^T\right)]|_2 \geq \frac{n}{m}\frac{p}{(1 - p)}.$$
\end{proposition}
\begin{remark}
For graph-based assignment schemes, we have $m = dn/2$, yielding 
$$|\mathbb{E}[(\alpha - \mathbbm{1})(\alpha - \mathbbm{1})^T]|_2 \geq \frac{2p}{d(1 - p)}.$$
\end{remark}
\begin{proof}
We can assume without loss of generality that $\hat{w} = \mathbbm{1}$. To see this, observe that we can scale the $j$th column of $A$ by a factor of $\hat{w}_j$. Because we ultimately care about the normalized quantity $\Abar$, we can also assume that $A$ is scaled such that $c = 1$.

We can calculate the covariance of $\alpha$ using the independence of machine failures: \begin{equation}
\begin{split}
\mathbb{E}[(\alpha - \mathbb{E}[\alpha])(\alpha - \mathbb{E}[\alpha])^T] &= \mathbb{E}[(Aw - A\mathbb{E}[w])(Aw - A\mathbb{E}[w])^T] \\
&= A\Cov(w)A^T \\
&= A\left(p(1 - p)I\right)A^T = p(1 - p)AA^T.
\end{split}  
\end{equation}

Hence by the circular law of trace, $$\mathbb{E}\left[|\alpha - \mathbbm{1}|^2_2\right] = \mathbb{E}\left[\Tr\left((\alpha - \mathbbm{1})(
\alpha - \mathbbm{1})^T\right)\right] = \Tr(p(1 - p)AA^T) = p(1 - p)|A|_F^2.$$

We know $A$ has at most $dn$ non-zero entries, and that $$\mathbbm{1}^TA\mathbbm{1} = \mathbbm{1}^TA\hat{w} = \mathbbm{1}^TA\frac{\mathbb{E}[w]}{(1 - p)} = \mathbbm{1}^T\frac{\mathbb{E}[\alpha]}{(1 - p)} = \frac{n}{1 - p},$$ because $\mathbb{E}[\alpha] = \mathbbm{1}$. 
To minimize $|A|_F,$ subject to this condition, we should set all non-zero entries of $A$ equal to $\frac{1}{d(1 - p)}$. This yields $|A|_F^2 = \frac{n}{d(1 - p)^2},$ and so $$\mathbb{E}\left[|\alpha - \mathbbm{1}|^2_2\right] \geq \frac{pn}{d(1 - p)}.$$

Now $$|AA^T|_2 \geq \frac{1}{m}\mathbbm{1}^TA^TA\mathbbm{1} = \frac{1}{m}\frac{n}{(1 - p)^2}.$$

The proposition follows because $\Cov(\alpha) = p(1 - p)AA^T.$
\end{proof}

\Cref{optfixed} compares fixed decoding to optimal decoding for expander graph-based schemes.

\begin{table}
\begin{center}
\begin{tabular}{| c | c | c| }
  \hline 
  Decoding Algorithm    & $\frac{1}{n}\mathbb{E}\left[|\Abar - \mathbbm{1}|_2\right]$ & $|\Cov(\Abar)|_2$  \\
  \hline
 Fixed Decoding (Lower bound) & $\frac{p}{d(1 - p)}$  & $\frac{2p}{d(1 - p)}$  \\
  \hline 
  Optimal Decoding (Upper bound) & $p^{d - o(d)}$ & $\log^2(n)p^{2d - o(d)}$ \\
  \hline 
\end{tabular}
\end{center}
\caption{Comparison of Optimal to Fixed Coefficient Decoding for Expander-Graph Based Schemes}\label{optfixed}
\end{table}
\begin{proposition}[Lower Bound for any decoding algorithm]\label{lb}
Consider any assignment scheme $A$ with $m$ machines $n$ data blocks with $n \leq nnz(A) \leq dn$. Suppose we use a some decoding algorithm that yields an unbiased gradient, that i, $$\mathbb{E}[\alpha] = c\mathbbm{1}$$ for some $c$. Then
$$\frac{1}{n}\mathbb{E}\left[|\Abar - \mathbbm{1}|_2^2\right] \geq \frac{p^d}{1 - p^d}.$$
\end{proposition}
\begin{proof}
Without loss of generality, we can scale the decoding coefficients $\alpha$ such $\mathbb{E}[\alpha] = \mathbbm{1}$. For each data block $i \in [n]$, let $d_i$ be its replication factor. Then with probability at least $p^{d_i}$, all machines holding the $i$th block straggle. When this occurs, we must have $\alpha_i = 0$. Now because we have $\mathbb{E}[\alpha_i] = 1$, we must have $$\mathbb{E}[(\alpha_i - 1)^2] \geq p^{d_i} + (1 - p^{d_i})\left(\frac{1}{1 - p^{d_i}} - 1\right)^2 = \frac{p^{d_i}}{1 - p^{d_i}}.$$

Hence $$\sum_{i = 1}^n\mathbb{E}[\alpha_i] \geq \sum_{i = 1}^n{\frac{p^{d_i}}{1 - p^{d_i}}} = n - \sum_{i = 1}^n{\frac{1}{1 - p^{d_i}}}.$$

We know that $\sum_i{d_i} \leq dn$, and hence this value is minimized when we have all $d_i = d$. Plugging this in yields the proposition.
\end{proof}
\begin{remark}
The proof of this lower bound holds even if the distributed algorithm uses a more complicated coding strategy than described in the introduction. This includes for example non-linear coding of the gradients or coordinate-wise coding of the gradients, which involve multiplying the gradients by a matrix as done in \cite{Ye}.
\end{remark}

\section{Convergence with biased assignment schemes}\label{apx:debias}

\begin{proposition}\label{debias}
Suppose there exists some assignment matrix $A$ with computational load $\ell$ on $m$ machines and $N$ data blocks, and decoding vector strategy $w$ and corresponding $\alpha$ such that \begin{equation}\label{eps1}\mathbb{E}\left[|\alpha - \mathbbm{1}|_2^2\right] \leq \epsilon N.\end{equation} There there exists some assignment matrix $\hat{A}$ with computational load at most $2\ell$ on $m$ machines and $N$ data blocks and decoding strategy $\hat{w}$ and corresponding $\hat{\alpha}$ such that 
$$\mathbb{E}\left[|\hat{\alpha} - \mathbbm{1}|_2^2\right] \leq \frac{2\epsilon}{(1 - \sqrt{2\epsilon})^2} N$$
 and $$\mathbb{E}[\hat{\alpha}] = \mathbbm{1}.$$
 
\end{proposition}

\begin{proof}
Let $\delta = 1 - \sqrt{2\epsilon}$, and let $S = \{i \in [N]: \mathbb{E}[\alpha_i] \geq \delta\}.$ Then by \Cref{eps1}, we must have 
$$|S| \geq N\left( 1 - \frac{\epsilon}{(1 - \delta)^2} \right) = \frac{N}{2}. $$

Let $s = |S|$, and $t = N - s$. Without loss of generality, we assume $S = [s]$. Let $A_S$ be the matrix of $A$ containing all rows in $S$. Let $D$ be the $s \times s$ diagonal matrix with entries $D_{ii} = \frac{1}{\mathbb{E}[A_Sw]_i}$. Define $\hat{A}$ to be the $N \times m$ matrix $DA_S$ concatenated with the first $t$ rows of $DA_S$ vertically: $$\hat{A} := [(DA_S)^T | (DA_S)^T[:t]]^T,$$ such that $\mathbb{E}[\hat{A}w] = \mathbbm{1}.$ Furthermore, since we have just scaled the rows of $A_S$, the replication factor of $\hat{A}$ is most $d$, but because of the concatenation, each machine may store at most twice as many data blocks as before, so the computational load of $\hat{A}$ is at most $2\ell$. For any straggler pattern, we set the decoding coefficients $\hat{w}$ to be equivalent to the coefficients $w$ used to decode for the assignment matrix $A$. Then we set $\hat{\alpha} = \hat{A}\hat{w}$.

Now for $i \in S$, we have 
\begin{align*}
\mathbb{E}[(\hat{\alpha_i} - 1)^2] &= \mathbb{E}[(\hat{\alpha}_{i + N/2} - 1)^2] \\
&=  \mathbb{E}\left[\left(\frac{\alpha_i}{\mathbb{E}[\alpha_i]} - 1\right)^2\right]\\
&= \left(\frac{1}{\mathbb{E}[\alpha_i]}\right)^2\mathbb{E}\left[\left(\alpha_i - \mathbb{E}[\alpha_i]\right)^2\right] \\
&\leq \left(\frac{1}{\mathbb{E}[\alpha_i]}\right)^2\mathbb{E}\left[\left(\alpha_i - 1\right)^2\right] \\
&\leq \frac{1}{\delta^2}\mathbb{E}\left[\left(\alpha_i - 1\right)^2\right].
\end{align*}

 It follows that $$\sum_{i \in S}\mathbb{E}[(\hat{\alpha_i} - 1)^2] \leq
 \frac{1}{\delta^2} \sum_{i \in S} \mathbb{E}[(\alpha_i - 1)^2]
 \leq \frac{1}{\delta^2}N\epsilon,$$ 
 so 
 $$\sum_{i \in [N]}\mathbb{E}[(\hat{\alpha_i} - 1)^2] \leq 2\frac{1}{\delta^2}N\epsilon = \frac{2\epsilon}{(1 - \sqrt{2\epsilon})^2}N$$ as desired.
 \end{proof}

For any coding scheme achieving $\frac{1}{N}
\mathbb{E}[|\alpha - \mathbbm{1}|_2^2] \leq \zeta$, we can combine \Cref{debias} with \Cref{conv}
to obtain the following convergence result under random stragglers.
\begin{proposition}\label{conv_biased}
Let $f = \sum_i^{N}f_i$ be a $\mu$-strongly convex function with an $L$-Lipshitz gradient, and suppose each $f_i$ is convex, and all gradients $\nabla f_i$ are $L^\prime$-Lipshitz. Let $\theta^*$ be the minimizer of $f$, and define $\sigma^2 := \sum_i{|\nabla f_i(\theta^*)|_2^2}$.

Suppose there exists an possibly biased coding scheme with computational load $\ell$ such that $$\frac{1}{N}\mathbb{E}\left[|\alpha - \mathbbm{1}|_2^2\right] \leq \zeta,$$ for some $\zeta \leq 1/4$. Suppose we modify the coding scheme according to \Cref{debias} to produce a coding scheme with computational load at most $2\ell$ and perform gradient coding as in Algorithm~\ref{dgc}, but with $w$ chosen as in the original coding scheme. Let $p$ be the probability of a machine straggling.

Then for any desired accuracy $\epsilon$, we can choose some step size $\gamma$ such that after $$k = 2\log(\epsilon /\epsilon_0)\left(\frac{L}{\mu} + 8n\zeta\frac{L^\prime}{\mu} + \frac{\zeta\sigma^2}{\mu^2\epsilon}\right)$$ steps of gradient descent, we have
$$\mathbb{E}\left[|\theta_k - \theta^*|^2_2\right] \leq \epsilon,$$ where $\epsilon_0 = |\theta_0 - \theta^*|^2$.
\end{proposition}

\section{Proof of \Cref{exp_cov_lemma}}\label{apx:cov}
For the reader's convenience, we restate \Cref{exp_cov_lemma} below.
\covlemma*
\begin{proof}(\Cref{exp_cov_lemma})
For the third statement, we have (for any $i$)
\begin{equation}
\begin{split}
|\mathbb{E}[(\alpha - r\mathbbm{1})(\alpha - r\mathbbm{1})^T]|_2 \leq \sum_{j}{|\mathbb{E}[\alpha_i\alpha_j] - \mathbb{E}[\alpha_i]\mathbb{E}[\alpha_j]|}.
\end{split}
\end{equation}
For any $S \subset V(G)$, let $E_S$ be the event that $S$ is a connected component in $G(p)$. For every vertex $i$, let
$$\mathcal{E}(i) = \left\{E_S : S \subset V, i \in S \right\}.$$
For $E \in \mathcal{E}(i),$ define $\alpha_i(E)$ to be the value of $\alpha_i$ conditioned on $E$.
For any $i, j$, we have
\begin{equation}\label{expansion}
\begin{split}
|\mathbb{E}[\alpha_i\alpha_j] - \mathbb{E}[\alpha_i]\mathbb{E}[\alpha_j]| &= |\mathbb{E}[(1 - \alpha_i)(1 - \alpha_j)] - \mathbb{E}[(1 - \alpha_i)]\mathbb{E}[(1 - \alpha_j)]| \\
&= \left|\sum_{E_S \in \mathcal{E}(i), E_T \in \mathcal{E}(j), }{(\Pr[E_S \cap E_T] - \Pr[E_S]\Pr[E_T])(1 - \alpha_i(E_S))(1 - \alpha_j(E_T))}\right| \\
& \leq \sum_{\substack{E_S, E_T: \\ \alpha_i(E_S) \neq 1, \alpha_j(E_T) \neq 1}}{|\Pr[E_S \cap E_T] - \Pr[E_S]\Pr[E_T]|} \\
& \leq 2\sum_{\substack{E_S, E_T: \\ \alpha_i(E_S) \neq 1, \alpha_j(E_T) \neq 1 \\ \ \Pr[E_S]\Pr[E_T] > \Pr[E_S \cap E_T]}}{\Pr[E_S]\Pr[E_T]}.
\end{split}
\end{equation}

Observe that $$\Pr[E_S]\Pr[E_T] > \Pr[E_S \cap E_T]$$ precisely when the two events cannot occur simultaneously, that is, where $S \neq T$ and $S \cap T \neq \emptyset$. For any sets $S, T$, define the variable $$I(S, T) = \begin{cases}1 & S \neq T \text{ and } S \cap T \neq \emptyset\\ 0 & S = T \text{ or } S \cap T = \emptyset \end{cases}$$ 

Now fix $i$ and $|S| \leq k$ where $i \in S$.
For $j \in [n]$, let $A_j \subset \Aut(G)$ be the set of automorphisms of $G$ that map vertex $j$ to vertex $1$. Then: 
\begin{align*}
    \sum_j\sum_{\substack{E_T \in \mathcal{E}(j) \\ |T| \leq k}}{\Pr[E_T]I(S, T)} 
    &=\sum_j\sum_{\substack{E_T \in \mathcal{E}(j) \\ |T| \leq k}}{\mathbb{E}_{\sigma \in A_j}\Pr[E_{\sigma^{-1}(T)}]I(S, T)} \\
    &=n \mathbb{E}_{j \in [n]}\mathbb{E}_{\sigma \in A_j}\sum_{\substack{E_T \in \mathcal{E}(j) \\ |T| \leq k}}{\Pr[E_{\sigma^{-1}(T)}]I(S, T)} \\
    &=n \mathbb{E}_{j \in [n]}\mathbb{E}_{\sigma \in A_j}\sum_{\substack{E_T \in \mathcal{E}(1) \\ |T| \leq k}}{\Pr[E_{T}]I(S, \sigma(T))} \\
    &=n \mathbb{E}_{\sigma \in \Aut(G)}\sum_{\substack{E_T \in \mathcal{E}(1) \\ |T| \leq k}}{\Pr[E_{T}]I(S, \sigma(T))} \\
    &=n \sum_{\substack{E_T \in \mathcal{E}(1) \\ |T| \leq k}}{\Pr[E_{T}]\mathbb{E}_{\sigma \in \Aut(G)}[I(S, \sigma(T))]} 
\end{align*}
where in the first line we have used the fact that $\Pr[E_T] = \Pr[E_{\sigma(T)}]$ for any $\sigma \in \Aut(G)$; in the second line we have used linearity of expectations; in the third line we have re-indexed the sum; in the fourth line we have used the fact that, because $G$ is vertex-transitive, choosing a random $j \in [n]$ and then a random $\sigma \in S_j$ is the same as choosing a uniformly random $\sigma \in \Aut(G)$; and in the last line we have again used linearity of expectation.  
%\mkw{Also, to be extra pedantic, the reason that $\mathbb{E}_j \mathbb{E}_{\sigma \in A_j}$ is the same as $\mathbb{E}_{\sigma \sim \Aut(G)}$ is because all of the $|A_j|$'s have the same size, which follows from the fact that, if $\sigma$ is the automorphism mapping $j$ to $j'$, $\pi \mapsto \sigma \pi$ is a bijection from $A_{j'}$ to $A_{j}$.}\Note{Yay, this looks good, I like what you did with the sets $A_j$! Probably doesn't matter if its pedantic because no one will read it anyways.} 

Then, we have:
\begin{equation}\label{auto}
    \begin{split}
     \sum_j\sum_{\substack{E_T \in \mathcal{E}(j) \\ |T| \leq k}}{\Pr[E_T]I(S, T)} 
     &= n \sum_{\substack{E_T \in \mathcal{E}(1) \\ |T| \leq k}}\Pr[E_T]\Pr_{\sigma \sim \Aut(G)}[I(S,\sigma(T)) = 1]\\
      &\leq n \sum_{\substack{E_T \in \mathcal{E}(1) \\ |T| \leq k}}\Pr[E_T]\Pr_{\sigma \sim \Aut(G)}[S \cap \sigma(T) \neq \emptyset] \\
     &\leq  n \sum_{\substack{E_T \in \mathcal{E}(1) \\ |T| \leq k}}\Pr[E_T]\mathbb{E}_{\sigma \sim \Aut(G)}[|S \cap \sigma(T)|].
    \end{split}
\end{equation}

By the vertex transitivity of the graph $G$, for any vertex $u$, the distribution of $\sigma(u)$ is uniform on $V$ when $\sigma$ is drawn uniformly from $\Aut(G)$. Hence $$\mathbb{E}_{\sigma \sim \Aut(G)}[|S \cap \sigma(T)|] = \frac{|S||T|}{n}.$$
It follows from \Cref{auto} that
\begin{equation}\label{auto2}
\begin{split}
     \sum_j\sum_{\substack{E_T \in \mathcal{E}(j) \\ |T| \leq k}}{\Pr[E_T]I(S, T)} & \leq  n \sum_{\substack{E_T \in \mathcal{E}(1) \\ |T| \leq k}}\Pr[E_T]\frac{|S||T|}{n} \\
     &\leq k|S|\Pr[|C(1)| \leq k]\\ &\leq k|T|\left(t + \frac{6}{n}\right),
\end{split}
\end{equation}
where the last inequality follows from \Cref{bipartite}, which guarantees that any arbitary vertex will be in a giant component of size greater than $n/2$ with probability at least $1 - t - \frac{6}{n}$.
We will need the following claim which directly follows from the second statement in \Cref{bipartite}.
\begin{claim}\label{ne1_claim}
For any vertex $i$, the probability that $\alpha_i \neq 1$ and $i$ is in a component of size greater than $k$ is at most $\frac{6}{n}$.
\end{claim}
Returning to \Cref{expansion} and summing over all $j$, we have 
\begin{equation}
\begin{split}
\sum_j{|\mathbb{E}[\alpha_i\alpha_j] - \mathbb{E}[\alpha_i]\mathbb{E}[\alpha_j]|} &\leq \sum_j 2\sum_{\substack{E_S \in \mathcal{E}(i), E_T \in \mathcal{E}(j) \\ \alpha_i(E_S) \neq 1, \alpha_j(E_T) \neq 1 \\ \Pr[E_S]\Pr[E_T] > \Pr[E_S \cap E_T]}}{\Pr[E_S]\Pr[E_T]}\\
&=2\sum_{\substack{E_S \in \mathcal{E}(i)\\ |S| \leq k}}\Pr[E_S]\sum_j\sum_{\substack{E_T \in \mathcal{E}(j)\\ |T| \leq k}}{\Pr[E_T]I(S, T)} + 2\sum_j\sum_{\substack{E_S \in \mathcal{E}(i), E_T \in \mathcal{E}(j) \\ \alpha_i(E_S) \neq 1, \alpha_j(E_T) \neq 1 \\ \max(|S|, |T|) > k}}{\Pr[E_S]\Pr[E_T]}\\
&\leq 2\sum_{\substack{E_S \in \mathcal{E}(i)\\ |S| \leq k}}{\Pr[E_S]|S|k\left(t + \frac{6}{n}\right)} + 2\sum_j\sum_{\substack{E_S \in \mathcal{E}(i), E_T \in \mathcal{E}(j) \\ \alpha_i(E_S) \neq 1, \alpha_j(E_T) \neq 1 \\ \max(|S|, |T|) > k}}{\Pr[E_S]\Pr[E_T]}\\
&\leq 2k^2\left(t + \frac{6}{n}\right)^2 + 4n\cdot \frac{6}{n} \\
&\leq 2k^2\left(t + \frac{6}{n}\right)^2 + 24,
\end{split}
\end{equation}
where in the second inequality we used \Cref{auto}, and in the second to last line we used \Cref{ne1_claim}.
\end{proof}
\section{Bounds on Moments of Binomials}\label{apx:binom}
For completeness, in this appendix we prove \Cref{prop:binomialfact}, which we restate below.  
\binomfact*
\begin{proof}
For $i \in [n]$, let $X_i$ and $Y_i$ be i.i.d. Bernoulli variables with parameter $q$. Then
\begin{equation}
\mathbb{E}\left[\left(\text{Binomial}(n, q)  - nq \right)^c\right] = \mathbb{E}\left[\left(\sum_i{(X_i - q)}\right)^c\right] = \mathbb{E}\left[\left(\sum_i{X_i} - \sum_i{\mathbb{E}[Y_i]}\right)^c\right] \leq \mathbb{E}\left[\left(\sum_i{(X_i - Y_i)}\right)^c\right],
\end{equation} where the inequality follows by Jensen's inequality.

Let $Z_i \sim \text{Bernoulli}(2q(1 - q))$ be i.i.d. random variables such that $(X_i - Y_i) \sim r_iZ_i$, where $r_i$ are i.i.d. Rademacher variables. Let $g_i \sim \mathcal{N}(0, 1)$ be i.i.d. Gaussians with variance $1$.

Then \begin{equation}
    \mathbb{E}\left[\left(\sum_i{r_iZ_i}\right)^c\right] \leq \mathbb{E}\left[\left(\sum_i{g_iZ_i}\right)^c\right],
\end{equation}
because the even moments of a Gaussian are at least as large as those of a Rademacher random variable, and all the odd moments are zero for both.  Now because for all positive integers $k$ we have $\mathbb{E}[Z_i^k] = 2q(1 - q)$, by comparing every moment, we see that
$$\sum_i{g_iZ_i} \sim 2q(1 - q)\sum_i{g_i} \sim 2q(1 - q)\mathcal{N}(0, n),$$ and so 
\begin{equation}
    \mathbb{E}\left[\left(\sum_i{g_iZ_i}\right)^c\right] = \left(4nq^2(1-q)^2\right)^{c/2}\frac{c!}{2^{\frac{c}{2}(c/2)!}} \leq \left(4nq^2(1-q)^2c\right)^{c/2},
\end{equation}
where the inequality follows from the fact that $\left(\frac{k!}{2^{\frac{k}{2}}(k/2)!}\right)^{1/k} \leq \sqrt{k}$ by Stirling's formula. Finally, $$\left(4nq^2(1-q)^2c\right)^{c/2} \leq \left(2q\sqrt{nc}\right)^c,$$ from which the proposition follows.
\end{proof}

\section{Proof of Proposition~\ref{conv} and \Cref{convcor}}\label{apx:conv}

We restate Proposition~\ref{conv} here:
\rstconv*
Our proof of \Cref{conv} will use the following known lemmas.
\begin{lemma}[\cite{trace}]\label{fact}
If matrices $A$ and $B$ are PSD, then $$\Tr(AB) \leq |A|_2\Tr(B).$$
\end{lemma}

\begin{lemma}[Co-coercivity lemma in \cite{SGD}]\label{coercive}
For a smooth function $f$ whose gradient has Lipschitz constant $L$,
$$|\nabla f(x) - \nabla f(y)|_2^2 \leq L \langle{x- y, \nabla f(x) - \nabla f(y)}\rangle.$$
\end{lemma}
Given these, we can prove \Cref{conv}. This proof is inspired by the stochastic gradient descent convergence proof in  \cite{SGD}).
\begin{proof}(\Cref{conv})
For convenience, let $g_i(x) = \nabla f_i(x),$ and let $G(x)$ be the matrix whose $i$th column is $g_i(x)$. Let $y_k = x_k - x^*.$ Let $\rho$ be a uniformly random permutation and $\beta \sim P_{\beta}$.
\begin{equation}\label{iterate}
\begin{split}
|y_{k + 1}|_2^2 &= \left|y_k - \gamma G(x_k)\rho^{-1}(\beta) \right|_2^2 \\
&= |y_k|_2^2 - 2\gamma y_k^TG(x_k)\rho^{-1}(\beta) + \gamma^2|G(x_k)\rho^{-1}(\beta)|_2^2 \\
&\leq |y_k|_2^2 - 2\gamma y_k^TG(x_k)\rho^{-1}(\beta) + 2\gamma^2|(G(x_k) - G(x^*))\rho^{-1}(\beta)|_2^2 + 2\gamma^2|G(x^*)\rho^{-1}(\beta)|_2^2.
\end{split}
\end{equation}

\begin{claim}
For any $\rho \in \mathcal{S}_n$, $$\mathbb{E}_{\beta}[|(G(x_k) - G(x^*))\rho^{-1}(\beta)|_2^2] \leq \left(sL^\prime + L\right)y_k^T\nabla f(x).$$
\end{claim}
\begin{proof}
Using the fact that $\mathbb{E}[\beta] = \mathbbm{1}$, the circular law of trace and \Cref{fact}, we have,
\begin{equation}\label{gradbd}
\begin{split}
\mathbb{E}_\beta[|(G(x_k) - G(x^*))\rho^{-1}(\beta)|_2^2] &= \mathbb{E}_\beta[(\rho^{-1}(\beta) - \mathbbm{1})^T(G(x_k) - G(x^*))^T(G(x_k) - G(x^*))(\rho^{-1}(\beta) - \mathbbm{1})]\\ 
&\qquad + \mathbbm{1}^T(G(x_k) - G(x^*))^T(G(x_k) - G(x^*))\mathbbm{1} \\
&= \Tr(\mathbb{E}_\beta[(\rho^{-1}(\beta)- \mathbbm{1})(\rho^{-1}(\beta) - \mathbbm{1})^T](G(x_k) - G(x^*))^T(G(x_k) - G(x^*)))\\
&\qquad + |\nabla f(x_k) - \nabla f(x^*)|_2^2\\
&\leq |\mathbb{E}_\beta[(\rho^{-1}(\beta)- \mathbbm{1})(\rho^{-1}(\beta) - \mathbbm{1})^T]|_2\Tr((G(x_k) - G(x^*))^T(G(x_k) - G(x^*)))\\
&\qquad + |\nabla f(x_k) - \nabla f(x^*)|_2^2\\
&= s\sum_i{|g_i(x_k) - g_i(x^*)|^2} + |\nabla f(x_k) - \nabla f(x^*)|_2^2.
\end{split}
\end{equation}
We now use \Cref{coercive} and the convexity of the $f_i$ to bound
$$\sum_i{|g_i(x_k) - g_i(x^*)|^2} \leq \sum_i{L^\prime\langle{y_k, g_i(x_k) - g_i(x^*)}\rangle} = L^\prime\langle{y_k, \nabla f(x_k) - \nabla f(x^*)}\rangle,$$ and similarly

$$|\nabla f(x_k) - \nabla f(x^*)|_2^2 \leq L\langle{y_k, \nabla f(x_k) - \nabla f(x^*)}\rangle.$$

Plugging these bounds into \Cref{gradbd} yields the claim.
\end{proof}
Returning to \Cref{iterate} and taking expectations with respect to $\beta^{(k)}$, conditional on $\rho$ we have
\begin{equation}
\mathbb{E}_{\beta^{(k)}}\left[|y_{k + 1}|_2^2 \bigm| \rho\right] \leq |y_k|_2^2 - 2\gamma \langle{y_k, \nabla f(x)}\rangle +
2\gamma^2\left(sL^\prime + L\right)\langle{y_k, \nabla f(x)}\rangle + 2\gamma^2\mathbb{E}_{\beta^{(k)}}\left[\left|G(x^*)\rho^{-1}(\beta)\right|_2^2 \biggm| \rho\right]
\end{equation}
Using the strong convexity of $f$ and the assumption $\gamma \leq \frac{1}{sL^\prime + L}$, we have, using $\mathbb{E}[ \beta^{(k)}] = \mathbbm{1}$,
\begin{equation}\label{ykbound}
\begin{split}
\mathbb{E}_{\beta^{(k)}}\left[|y_{k + 1}|_2^2 \bigm| \rho\right] &\leq |y_k|^2 - 2\gamma\mu\left(1 - \gamma(sL^\prime  + L)\right)|y_k|^2 + 2\gamma^2\mathbb{E}_{\beta^{(k)}}\left[\left|G(x^*)\rho^{-1}(\beta^{(k)})\right|_2^2 \biggm| \rho\right]. \\
&= |y_k|^2\left(1 - 2\gamma\mu\left(1 - \gamma(sL^\prime  + L)\right)\right) + 2\gamma^2\mathbb{E}_{\beta^{(k)}}\left[\left|G(x^*)\rho^{-1}(\beta^{(k)})\right|_2^2 \biggm| \rho\right].
\end{split}
\end{equation}

We will bound the second term in expectation over $\rho$ using the next claim.
\begin{claim}
$$\mathbb{E}_{\rho \sim \mathcal{S}_n, \beta}\left[|G(x^*)\rho^{-1}(\beta)|_2^2\right] \leq r\left(1 + \frac{1}{n - 1}\right)\sigma^2.$$
\end{claim}
\begin{proof}
Recall that because $x^*$ is optimal, $G(x^*)\mathbbm{1} = 0$. Now
\begin{equation}
\begin{split}
 |G(x^*)\rho^{-1}(\beta)|_2^2 &= \rho^{-1}(\beta)^TG(x^*)^TG(x^*)\rho^{-1}(\beta) \\
 &= \Tr(\rho^{-1}(\beta)\rho^{-1}(\beta)^TG(x^*)^TG(x^*)),
\end{split}
\end{equation} so 
\begin{equation}
\mathbb{E}_{\rho \sim \mathcal{S}_n, \beta}[|G(x^*)\rho^{-1}(\beta)|_2^2] = \Tr(\mathbb{E}_{\rho \sim \mathcal{S}_n, \beta}[\rho^{-1}(\beta)\rho^{-1}(\beta)^T]G(x^*)G(x^*))
\end{equation}

Now because $\rho$ is chosen randomly, the matrix $\mathbb{E}_{\rho \sim \mathcal{S}_n, \beta}[\rho^{-1}(\beta)\rho^{-1}(\beta)^T]$ has equal diagonal entries and equal off-diagonal entries. The diagonal entries equal
$$\frac{1}{n}\Tr(\mathbb{E}[\beta\beta^T]) = 1 + \frac{1}{n}\Tr(\mathbb{E}[(\beta - \mathbbm{1})(\beta - \mathbbm{1})^T]) = 1 + r,$$
while the off diagonal entries equal
\begin{equation*}
\begin{split}
\frac{1}{n - 1}\left(\frac{1}{n}\mathbbm{1}^T\mathbb{E}_{\beta}[\beta\beta^T]\mathbbm{1}- \frac{1}{n}\Tr(\mathbb{E}[\beta\beta^T])\right) &= \frac{1}{n - 1}\left(n + \frac{1}{n}\mathbbm{1}^T\mathbb{E}[(\beta - \mathbbm{1})(\beta - \mathbbm{1})^T]\mathbbm{1} - \frac{1}{n}\Tr(\mathbb{E}[\beta\beta^T])\right)\\ &\geq 1 - \frac{r}{n - 1}.
\end{split}
\end{equation*}
Hence
\begin{equation}
\begin{split}
\mathbb{E}_{\rho \sim \mathcal{S}_n, \beta}[\rho^{-1}(\beta)\rho^{-1}(\beta)^T] = aI + b\mathbbm{1}\mathbbm{1}^T,  
\end{split}
\end{equation}
where $a \leq r\left(1 + \frac{1}{n - 1}\right)$, and $b  \geq 1 - \frac{r}{n - 1}$.

Plugging this in, we have
\begin{equation}
\begin{split}
\mathbb{E}_{\rho \sim \mathcal{S}_n, \beta}[|G(x^*)\rho^{-1}(\beta)|_2^2] &= \Tr\left(\left(aI + b\mathbbm{1}\mathbbm{1}^T\right)G(x^*)^TG(x^*)\right) \\
&= a\Tr(G(x^*)^TG(x^*)) + b|G(x^*)\mathbbm{1}|_2^2 \\
&\leq r\left(1 + \frac{1}{n - 1}\right)\sigma^2,
\end{split}
\end{equation}
because $\Tr(G(x^*)^TG(x^*)) = \sum_i{|f_i(x^*)|^2} = \sigma^2,$ and $G(x^*)\mathbbm{1} = 0$.
\end{proof}

Recursively applying the bound in \Cref{ykbound} and taking the expectation over all $\beta^{(k)}$ and $\rho$ yields the proposition:
\begin{equation}
\begin{split}
\mathbb{E}_{\{\beta^{(j)}: j < k\}, \rho}\left[|y_{k}|_2^2\right] &\leq \left(1 - 2\gamma\mu\left(1 - \gamma(sL^\prime  + L)\right)\right)^k|y_0|_2^2 \\
&\qquad+ 2\gamma^2r\left(1 + \frac{1}{n - 1}\right)\sigma^2\sum_{i = 1}^{k - 1}{\left(1 - 2\gamma\mu\left(1 - \gamma(sL^\prime  + L)\right)\right)^j}\\
&\leq \left(1 - 2\gamma\mu\left(1 - \gamma(sL^\prime  + L)\right)\right)^k|y_0|_2^2 + \frac{\gamma r\left(1 + \frac{1}{n - 1}\right) \sigma^2}{\mu\left(1 - \gamma(sL^\prime  + L)\right)}.
\end{split}
\end{equation}
\end{proof}

Next we prove Corollary~\ref{convcor}, restated below.
\rstconvcor*
\begin{proof}
We plug the choice of $\gamma$ into \Cref{conv_bd} of \Cref{conv}.
The second summand in \Cref{conv_bd} is bounded by
\begin{equation}
\frac{\gamma r\left(1 + \frac{1}{n - 1}\right) \sigma^2}{\mu\left(1 - \gamma(sL^\prime  + L)\right)} \leq \frac{\epsilon}{2}.
\end{equation}
For the first term in \Cref{conv_bd} to be less than $\frac{\epsilon}{2}$, we must have
\begin{equation}\label{klb}
k \geq \frac{\log\left(\frac{\epsilon}{2\epsilon_0}\right)}{\log\left(1 - 2\gamma\mu\left(1 - \gamma(sL^\prime  + L)\right)\right)}.
\end{equation}
Note that $$2\gamma\mu\left(1 - \gamma(sL^\prime  + L)\right) \in (0, 1)$$ because our choice of $\gamma$ satisfies $\gamma < \frac{1}{2(sL^{\prime} + L)} \leq \frac{1}{2\mu}$.
Plugging in the choice of $\gamma$ into the denominator of \Cref{klb}, we have 
\begin{equation}
\begin{split}
\log\left(1 - 2\gamma\mu\left(1 - \gamma(sL^\prime  + L)\right)\right) &= \log\left(1 - \frac{\mu^2\epsilon\left(\mu \epsilon (sL^\prime  + L) + 2r\left(1 + \frac{1}{n - 1}\right)\sigma^2\right)}{2\left(\mu \epsilon (sL^\prime  + L) + r\left(1 + \frac{1}{n - 1}\right)\sigma^2\right)^2}\right)\\
& \leq - \frac{\mu^2\epsilon\left(\mu \epsilon (sL^\prime  + L) + 2r\left(1 + \frac{1}{n - 1}\right)\sigma^2\right)}{2\left(\mu \epsilon (sL^\prime  + L) + r\left(1 + \frac{1}{n - 1}\right)\sigma^2\right)^2} \\
& \leq -\frac{\mu^2\epsilon}{2\left(\mu \epsilon (sL^\prime  + L) + r\left(1 + \frac{1}{n - 1}\right)\sigma^2\right)},
\end{split}
\end{equation}
where the first inequality follows by using $\log(1 - x) \leq - x$ for $x \in (0, 1)$. 

It follows that that the value of $k$ in the corollary satisfies \Cref{klb}, which proves the result.
\end{proof}

\section{Proof of Propostion~\ref{prop:formal_adv_conv}}\label{apx:adv_conv}

We restate the Propostion:
\rstadv*

The proof of this proposition relies on the following key lemma. 

\begin{lemma}\label{lem:adv_conv}
Consider the setting of \Cref{prop:formal_adv_conv}.
For any step size $\gamma$, we have
\begin{equation}\label{prop_result}
|x_{k + 1} - x^*|_2^2
\leq |x_k - x^*|_2^2\left(1  - \left(2 - \frac{\gamma(L^2 + 2rLL^\prime + 4r^2(L^\prime)^2)}{a\mu}\right)a\gamma\mu\right) + |x_k - x^*|_2\left(2\gamma r \sigma\right)(1 + \gamma L)
+ 4\gamma^2r^2\sigma^2.
\end{equation}

\end{lemma}
We begin by proving the proposition using the lemma. 

\begin{proof}(\Cref{prop:formal_adv_conv})
Define $y_t := x_t - x^*$.
Assume that after the $t$'th iteration, the convergence criterion has not been met; that is $$|y_t|_2^2 > (1 + \epsilon)^2\frac{r^2\sigma^2}{a^2\mu^2},$$ where in the first inequality, we used that $\gamma L < \frac{\epsilon a}{6}$.
This implies that 
\begin{equation}|y_t|_2(2\gamma r\sigma)(1 + \gamma L) \leq |y_t|_2(2\gamma r\sigma)\left(1 + \frac{\epsilon a}{6}\right)\leq |y_t|_2^2\left(\frac{2\gamma a \mu}{1 + \epsilon}\right)\left(\frac{6 + \epsilon a}{6}\right).
\end{equation}
Then \cref{prop_result} and $a \leq 1$ and $\epsilon \leq 1$ imply that
\begin{equation}\label{eq:iterate}
\begin{split}
|y_{t + 1}|_2^2
&\leq |y_t|_2^2\left(1  - a\gamma\mu\left(2 - \frac{\epsilon}{6} - \frac{2(6 + \epsilon)}{6(1 + \epsilon)}\right)\right)
+ 4\gamma^2r^2\sigma^2 \\
& \leq |y_t|_2^2\left(1  - a\gamma\mu\left(\frac{3\epsilon}{2(1 + \epsilon)}\right)\right)
+ 4\gamma^2r^2\sigma^2.
\end{split}
\end{equation}

Thus applying the bound of \cref{eq:iterate} recursively, after $k$ iterations, either \cref{threshold} has been achieved, or \begin{equation}\label{recurse}
\begin{split}
|y_k|_2^2 &\leq |y_0|_2^2\left(1  - a\gamma\mu\left(\frac{3\epsilon}{2(1 + \epsilon)}\right)\right)^k + 4\gamma^2r^2\sigma^2\sum_{j = 1}^{k - 1}{\left(1  - a\gamma\mu\left(\frac{3\epsilon}{2(1 + \epsilon)}\right)\right)^j} \\
&\leq |y_0|_2^2\left(1  - a\gamma\mu\left(\frac{3\epsilon}{2(1 + \epsilon)}\right)\right)^k + \frac{4\gamma r^2\sigma^2}{a\mu\left(\frac{3\epsilon}{2(1 + \epsilon)}\right)}.
\end{split}
\end{equation} 
Here we used the fact that $a\gamma\mu\left(\frac{3\epsilon}{2(1+ \epsilon)}\right) < 1$ to contract the sum.

For the second term of \cref{recurse}, plugging in the value for $\gamma$ yields 
\begin{equation}\begin{split}
\frac{4\gamma r^2\sigma^2}{a\mu\left(\frac{3\epsilon}{2(1 + \epsilon)}\right)} &= (1 + \epsilon)^2\frac{r^2\sigma^2}{a^2\mu^2}\left(\frac{8\gamma a \mu}{3\epsilon(1 + \epsilon)}\right)\\ 
&=  (1 + \epsilon)^2\frac{r^2\sigma^2}{a^2\mu^2}\left(\frac{8}{18(1 + \epsilon)}\right)\left(\frac{a^2\mu^2}{L^2 + 2rLL^\prime + 4r^2(L^\prime)^2}\right)\\
&\leq \frac{1}{2}(1 + \epsilon)^2\frac{r^2\sigma^2}{a^2\mu^2},
\end{split}
\end{equation}
which is half of the squared value in \cref{threshold}.

For the first term of \cref{recurse}, we have 
\begin{equation}
|y_0|_2^2\left(1  - a\gamma\mu\left(\frac{3\epsilon}{2(1 + \epsilon)}\right)\right)^k \leq |y_0|_2^2\exp\left(-ka\gamma\mu\left(\frac{3\epsilon}{2(1 + \epsilon)}\right)\right)
\end{equation}

Thus for $k$ larger than the right hand side of \cref{k_val}, we have 
\begin{equation}
\begin{split}
|y_0|_2^2\left(1  - a\gamma\mu\left(\frac{3\epsilon}{2(1 + \epsilon)}\right)\right)^k &\leq |y_0|_2^2\exp\left(-\log\left(\frac{2a^2\mu^2|y_0|_2^2}{(1 + \epsilon)^2r^2\sigma^2}\right)\right)\\ &= (1 + \epsilon)^2\frac{r^2\sigma^2}{2a^2\mu^2},
\end{split}
\end{equation}
which is half of the squared value in \cref{threshold}. This yields the result.
\end{proof}

Next we prove \Cref{lem:adv_conv}.
\begin{proof}(\Cref{lem:adv_conv})
For convenience, let $g_i(x) = \nabla f_i(x),$ and let $G(x)$ be the matrix whose $i$th column is $g_i(x)$. We abbreviate $G(x_k)$ by $G$ and $g_i(x_k)$ by $g_i$. Let $y_k := x_k - x^*$. Our gradient step \Cref{eq:adv_grad_step} guarantees that
\begin{equation}
|y_{k + 1}|_2^2 \leq  \max_{\beta: |\beta|_2 \leq r}|y_k - \gamma G(x_k)(\mathbbm{1} + \beta)|_2^2.
\end{equation}
By the method of Lagrange multipliers, the optimizer $\beta_*$ is the maximizer of 
\begin{equation}\label{eq:lagrange}
    |y_k - \gammaG\mathbbm{1}|_2^2 + \gamma^2\beta^TG^TG\beta - 2\gamma(y_k - \gammaG\mathbbm{1})^TG\beta + \lambda\beta^T\beta,
\end{equation}
for some negative $\lambda$. Setting the derivative of \cref{eq:lagrange} to zero and solving yields
\begin{equation}\label{eq:bstar}
    \beta_* = \gamma(\lambda I + \gamma^2G^TG)^{-1}G^T(y_k - \gammaG\mathbbm{1}).
\end{equation}
Clearly at the maximum, the constraint $|\beta|_2^2 = r^2$ will hold, so, setting the norm of the value in \cref{eq:bstar} equal to $r$ gives the following condition on $\lambda$:
\begin{equation}\label{lambda_sol}|\gamma(\lambda I + \gamma^2G^TG)^{-1}G^T(y_k - \gammaG\mathbbm{1})|_2 = r
\end{equation}
Plugging the value of $\beta_*$ for \cref{eq:bstar}, we have 

\begin{equation}\label{eq:13}
\begin{split}
\max_{\beta: |\beta|_2 \leq r}|y_k - \gamma G (\mathbbm{1} + \beta)|_2^2 &= |y_k - \gammaG\mathbbm{1}|_2^2 -2\beta_*^T(\lambda I + \gamma^2G^TG)\beta_* + \gamma^2\beta_*^T(G^TG)\beta_* \\
&= |y_k - \gammaG\mathbbm{1}|_2^2 -\beta_*^T(\lambda I + \gamma^2G^TG)\beta_* - \lambda r^2 \\
&\leq |y_k - \gammaG\mathbbm{1}|_2^2 - 2r^2\lambda.
\end{split}
\end{equation}

Now by \cref{lambda_sol}, we have \begin{equation}\frac{\gamma^2|G^T(y_k - \gammaG\mathbbm{1})|_2^2}{r^2} \geq \min\left((\lambda + \gamma^2\sigma_1(G^TG))^2, (\lambda + \gamma^2\sigma_n(G^TG))^2\right)
\end{equation} 

yielding \begin{equation}
\lambda \geq -\frac{\gamma|G^T(y_k - \gammaG\mathbbm{1})|_2}{r} - \gamma^2\sigma_1(G^TG).
\end{equation}

Plugging this in to \cref{eq:13} yields

\begin{equation}\label{eq:16}
\begin{split}
\max_{\beta: |\beta|_2 \leq r}|y_k - \gamma G (\mathbbm{1} + \beta)|_2^2 &\leq |y_k - \gammaG\mathbbm{1}|_2^2 + 2\gamma^2r^2\sigma_1(G^TG) + 2\gamma r|G^T(y_k - \gammaG\mathbbm{1})|_2.
\end{split}
\end{equation}

In the next three claims, we bound the quantities in this equation.

\begin{claim}\label{final_term}
$$|G^TG\mathbbm{1}|_2 \leq LL^\prime|y_k|_2^2 + L\sigma|y_k|_2.$$
\end{claim}
\begin{proof}
First notice that 
\begin{equation}
|G^TG\mathbbm{1}|_2 \leq |\left(G(x_k) - G(x_*)\right)^TG\mathbbm{1}|_2 + |G(x_*)^TG\mathbbm{1}|_2.
\end{equation}

Now

\begin{equation}
\begin{split}
|\left(G(x_k) - G(x_*)\right)^TG\mathbbm{1}|_2^2 &= \sum_i{(\left(g_i(x_k) - g_i(x_*)\right)^T(G \mathbbm{1}))^2} \\
&\leq |G\mathbbm{1}|_2^2\sum_i{|g_i(x_k) - g_i(x_*)|_2^2} \\
&\leq |G\mathbbm{1}|_2^2|y_k|_2^2(L^\prime)^2 \\
&\leq L^2(L^\prime)^2|y_k|_2^4.
\end{split}
\end{equation}

We also have 
\begin{equation}
\begin{split}
|G(x_*)^TG\mathbbm{1}|_2^2  &= \sum_i{(g_i(x_*))^T(G \mathbbm{1}))^2} \\
& \leq |G\mathbbm{1}|_2^2\sigma^2 \leq L^2|y_k|^2\sigma^2.
\end{split}
\end{equation}
Taking square roots and summing yields the claim.
\end{proof}

\begin{claim}
$$|G^T(y_k - \gammaG\mathbbm{1})|_2 \leq |y_k|_2\left(\sigma + \sqrt{L^\prime \mathbbm{1}^TG^Ty_k}\right)+ \gamma\left(LL^\prime|y_k|_2^2 + L\sigma|y_k|_2\right)$$
\end{claim}
\begin{proof}
First observe that 
\begin{equation}
\begin{split}
|G^T(y_k - \gammaG\mathbbm{1})|_2 &\leq |G^Ty_k|_2 + \gamma|G^TG\mathbbm{1}|_2\\
&\leq |\left(G^T(x_k) - G^T(x_*)\right)y_k|_2 + |G^T(x_*)y_k|_2  + \gamma|G^TG\mathbbm{1}|_2.
\end{split}
\end{equation}

Now
\begin{equation}
|G^T(x_*)y_k|_2^2 \leq |G^T(x_*)|_2^2|y_k|_2^2 \leq \Tr(G(x_*)G^T(x_*))|y_k|_2^2 = \sigma^2|y_k|_2^2.
\end{equation}

Also, 
\begin{equation}
\begin{split}|\left(G^T(x_k) - G^T(x_*)\right)y_k|_2^2
&=  \sum_i{((g_i(x_k)^T - g_i^T(x_*))y_k)^2} \\
& \leq L^\prime|y_k|_2^2\sum_i{(g_i^T(x_k) - g_i^T(x_*))y_k}\\
&= |y_k|^2_2\left(L^\prime\mathbbm{1}^TG^Ty_k\right),
\end{split}
\end{equation}
where the inequality holds because each $f_i$ is convex and hence $(g_i^T(x_k) - g_i^T(x_*))y_k \geq 0$ for all $i$, and we have also used the fact that $\sum_i{g_i}(x_*) = 0$.

Taking square roots and combining with \cref{final_term} yields the claim.
\end{proof}

\begin{claim}
$$\sigma_1(G^TG) \leq 2|y_k|^2_2(L^\prime)^2 + 2\sigma^2.$$
\end{claim}
\begin{proof}
\begin{equation}
\begin{split}
\sigma_1(G^TG) &\leq \Tr(G^TG) = \sum_i{|g_i|_2^2} \\
&\leq 2\sum_i{|g_i(x_k) - g_i(x_*)|_2^2} + 2\sum_i{|g_i(x_*)|_2^2} \\
&\leq 2|y_k|^2_2(L^\prime)^2 + 2\sigma^2.
\end{split}
\end{equation}
\end{proof}

Plugging in these claims to \cref{eq:16} yields:
\begin{equation}\label{plug}
\begin{split}
\max_{\beta: |\beta|_2 \leq r}|y_k - \gamma G (\mathbbm{1} + \beta)|_2^2 &\leq |y_k - \gammaG\mathbbm{1}|_2^2 + 2\gamma^2r^2\left(2\sigma^2 + 2(L^\prime)^2|y_k|^2_2\right) \\&+ 2\gamma r
\left(|y_k|_2\left(\sqrt{L^\prime \mathbbm{1}^TG^Ty_k} + \sigma \right) + \gamma\left(LL^\prime|y_k|_2^2 + L\sigma|y_k|_2\right) \right) \\
& = |y_k|_2^2 - 2\gamma y_k^TG\mathbbm{1} + \gamma^2|G\mathbbm{1}|_2^2 + 4\gamma^2r^2\left(\sigma^2 + (L^\prime)^2|y_k|^2_2\right) \\
& \qquad+ 2\gamma r
\left(|y_k|_2\left(\sqrt{L^\prime \mathbbm{1}^TG^Ty_k} + \sigma \right) + \gamma\left(LL^\prime|y_k|_2^2 + L\sigma|y_k|_2\right) \right) \\
& \leq |y_k|_2^2 - 2\gamma y_k^TG\mathbbm{1} + \gamma^2L^2|y_k|_2^2 + 4\gamma^2r^2\left(\sigma^2 + (L^\prime)^2|y_k|^2_2\right) \\
& \qquad + 2\gamma r
\left(|y_k|_2\left(\sqrt{L^\prime \mathbbm{1}^TG^Ty_k} + \sigma \right) + \gamma\left(LL^\prime|y_k|_2^2 + L\sigma|y_k|_2\right) \right)\\
& = |y_k|_2^2\left(1 + \gamma^2L^2 + 2\gamma^2rLL^\prime + 4\gamma^2r^2(L^\prime)^2\right) + |y_k|_2\left(2\gamma r \sigma\right)(1 + \gamma L) \\
&\qquad + 4\gamma^2r^2\sigma^2 + 2\gamma y_k^TG^T\mathbbm{1}\left(-1 + \frac{r\sqrt{L^\prime} |y_k|_2}{\sqrt{y_k^TG\mathbbm{1}}}\right)
\end{split}
\end{equation}

Recall that $a = 1 - \frac{r\sqrt{L^\prime}}{\sqrt{\mu}}$ and that $a > 0$. Then
\begin{equation}
-1 + \frac{r\sqrt{L^\prime}|y_k|_2}{\sqrt{y_k^TG\mathbbm{1}}} \leq -1 + \frac{r\sqrt{L^\prime} |y_k|_2}{\sqrt{\mu}|y_k|_2} = -a.
\end{equation}

It follows from \cref{plug} that 
\begin{equation}
\begin{split}
\max_{\beta: |\beta|_2 \leq r}|y_k - \gamma G (\mathbbm{1} + \beta)|_2^2
& \leq |y_k|_2^2\left(1  - 2a\gamma\mu + \gamma^2\left(L^2 + 2rLL^\prime + 4r^2(L^\prime)^2\right)\right) \\ &+ |y_k|_2\left(2\gamma r \sigma\right)(1 + \gamma L)
+ 4\gamma^2r^2\sigma^2.
\end{split}
\end{equation}
Thus for any $\gamma$,  we have
\begin{equation}\label{final2}
\begin{split}
|y_{k + 1}|_2^2
& \leq |y_k|_2^2\left(1  - \left(2 - \frac{\gamma(L^2 + 2rLL^\prime + 4r^2(L^\prime)^2)}{a\mu}\right)a\gamma\mu\right) + |y_k|_2\left(2\gamma r \sigma\right)(1 + \gamma L)
+ 4\gamma^2r^2\sigma^2.
\end{split}
\end{equation}

This concludes the proposition.
\end{proof}

\section{Step Sizes from Simulations}\label{apx:hp}

We chose step sizes using a grid search. For the experiments on the distributed cluster with replication factor $d = 3$, our grid search ranged over all step sizes $\gamma$ of the form $10^{-6}*(1.3^c)$ for $c \in 0, 1, 2, \cdots, 20$. For the simulated experiments with replication factor $d = 3$, we used linearly decreasing step sizes of the form $\gamma_t = \min\left(0.6, \frac{0.3\cdot 1.3^c}{t + 1}\right)$ for the best value of $c \in 1, 2, \cdots, 20$. In the following table, we show the best choice of $c$ in each experiment we ran.

\begin{table}[ht]
\caption{Step Size Chosen by Grid Search (Best choice of $c$ in grid search}\label{step_size_chart}
\begin{center}
\begin{tabular}{| c | c | c | c| c| c| c| c| }
\hline
Assignment & Decoding & \multicolumn{6}{c|}{Step Size} \\ 
Matrix & Algorithm & \footnotesize $p=0.05$ & \footnotesize $p=0.10$ & \footnotesize $p=0.15$ &\footnotesize  $p=0.20$ & \footnotesize $p=0.25$ & \footnotesize  $p=0.30$  \\ 
  \hline \hline
  $A_1$ & Optimal  & 9 & 4 & 9 & 9 & 4 & 9 \\
     \hline
  $A_1$ & Fixed & 6 & 6 & 8 & 0 & 0 & 6 \\
     \hline
  Uncoded & Ignore Stragglers & 1 & 2 & 7 & 6 & 7 & 0 \\
     \hline
  Expander of \cite{ExpanderCode} (d = 3) & Optimal & 4 & 3 & 0 & 7 & 2 & 5 \\
     \hline
  FRC of \cite{Tandon} (d = 3) & Optimal & 5 & 9 & 9 & 4 & 1 & 1 \\
       \hline
   $A_2$ & Optimal & 18 & 18 & 18 & 11 & 11 & 10 \\
        \hline
  $A_2$ & Fixed & 9 & 9 & 9 & 9 & 9 & 9 \\
    \hline
  Uncoded & \makecell{Ignore Stragglers \\(6x its)} & 1 & 1 & 1 & 1 & 2 & 2 \\
     \hline
  Expander of \cite{ExpanderCode} (d = 6) & Fixed &  9 &  8 & 8 & 8 & 8 &  8\\
     \hline
  FRC of \cite{Tandon} (d = 6) & Optimal & 19  & 18 & 18  & 10 & 10 & 9  \\
    \hline
\end{tabular}
\end{center}
\end{table}
\end{appendices}

\end{document}